\documentclass[twoside]{article}
\let\OriginalAddContentsLine\addcontentsline
\usepackage[preprint]{aistats2027}
\usepackage[letterpaper,hmargin=1.03in,vmargin=1.04in,includeheadfoot,
  headheight=12pt,headsep=9pt,footskip=15pt]{geometry}
\let\addcontentsline\OriginalAddContentsLine
\usepackage[T1]{fontenc}
\usepackage[utf8]{inputenc}
\usepackage{microtype}
\usepackage{amsmath,amssymb,amsfonts,amsthm,mathtools,bm}
\usepackage{graphicx,needspace,booktabs,array,tabularx,enumitem}
\usepackage[round]{natbib}
\usepackage{algorithm,algpseudocode}
\usepackage[hidelinks,bookmarksopen=true,bookmarksnumbered=true]{hyperref}
\usepackage{bookmark}
\hypersetup{
  pdfauthor={Jeongung Heo and Seonghyun Jeong},
  pdftitle={RoBART: Bayesian Additive Regression Trees with Tree-Specific Rotations}
}
\newtheorem{theorem}{Theorem}
\newtheorem{proposition}{Proposition}

\newtheorem{assumption}{Assumption}
\theoremstyle{plain}
\newtheorem{Bresult}{Theorem}[subsection]
\newtheorem{Btheorem}[Bresult]{Theorem}
\newtheorem{Bdefinition}[Bresult]{Definition}
\newtheorem{Bassumption}[Bresult]{Assumption}
\newtheorem{Blemma}[Bresult]{Lemma}
\newtheorem{Bproposition}[Bresult]{Proposition}
\newtheorem{Bcorollary}[Bresult]{Corollary}
\theoremstyle{remark}
\newtheorem{Bremark}[Bresult]{Remark}
\theoremstyle{plain}
\newtheorem{Acorollary}{Corollary}[section]
\newtheorem{Ctheorem}{Theorem}[section]

\newcommand{\R}{\mathbb R}
\newcommand{\E}{\mathbb E}
\newcommand{\SO}{\mathrm{SO}}
\newcommand{\St}{\mathrm{St}}
\newcommand{\cT}{\mathcal T}

\newcommand{\cL}{\mathcal L}
\newcommand{\cX}{\mathcal X}
\newcommand{\cI}{\mathcal I}
\newcommand{\cA}{\mathcal A}
\newcommand{\cC}{\mathcal C}
\newcommand{\ind}{\mathbf 1}
\newcommand{\1}{\mathbf{1}}
\newcommand{\norm}[1]{\left\lVert#1\right\rVert}
\newcommand{\abs}[1]{\left|#1\right|}
\newcommand{\trans}{\mathsf T}
\newcommand{\eps}{\varepsilon}

\newcommand{\IG}{\operatorname{IG}}
\newcommand{\TN}{\operatorname{TN}}
\newcommand{\dep}{\operatorname{dep}}
\newcommand{\inner}[2]{\left\langle #1,#2\right\rangle_n}
\newcommand{\Var}{\mathrm V}
\newcommand{\KL}{\mathrm{KL}}
\newcommand{\DX}{D_X}
\newcommand{\epsn}{\varepsilon_n}

\newcommand{\Frot}{\mathcal{F}^{\mathrm{rot}}}

\newcommand{\B}{\mathrm B}
\newcommand{\Faxis}{\mathcal F^{\mathrm{ax}}}
\newcolumntype{L}[1]{>{\raggedright\arraybackslash}p{#1}}

\newcommand{\Mainref}[1]{\ref{main:#1}}
\newcommand{\Mainlink}[2]{\hyperref[main:#1]{#2}}

\begin{document}

\runningtitle{RoBART: Bayesian Additive Regression Trees with Tree-Specific Rotations}
\runningauthor{Jeongung Heo, Seonghyun Jeong}

\twocolumn[
\aistatstitle{RoBART: Bayesian Additive Regression Trees with Tree-Specific Rotations}
\aistatsauthor{
  Jeongung Heo
  \And
  Seonghyun Jeong\textsuperscript{*}
}

\aistatsaddress{
  Yonsei University
}
] % \twocolumn[...] 끝

\begingroup
\renewcommand{\thefootnote}{\fnsymbol{footnote}}
\footnotetext[1]{Corresponding author: \texttt{sjeong@yonsei.ac.kr}}
\endgroup

\begin{abstract}
Bayesian additive regression trees (BART) can require many splits to approximate boundaries misaligned with the predictor axes. RoBART assigns each tree a rotation shared by all internal nodes, retaining axis-aligned splits in rotated coordinates and constant leaves. We jointly propose a Givens rotation sequence and cutpoints on the resulting grid by Metropolis–Hastings and establish reversibility with respect to the conditional posterior with leaf means integrated out. For additive functions with component-specific rotations and anisotropic H\"older smoothness, we prove posterior contraction in empirical $L_2$ distance and for the noise standard deviation. Under the stated prior, design, and grid conditions, with fixed numbers of predictors, trees, and components and no more components than trees, the rate is a sum of componentwise rates determined by smoothness and the number of rotated coordinates used. We also establish a posterior contraction lower bound showing that there exist functions for which RoBART adapts to the intrinsic dimension but axis-aligned BART does not.
\end{abstract}

\section{INTRODUCTION}
\label{main:sec:intro}

Bayesian additive regression trees (BART) represent a regression mean as a sum of regularized trees \citep{ChipmanGeorgeMcCulloch2010,HillLineroMurray2020}. Each internal node in standard BART splits on one predictor. These splits produce rectangular cells, so approximating an oblique boundary can require many splits.

Figure~\ref{main:fig:motivation}(a) compares posterior means with the true function $f(x)=\ind\{|x_1|>|x_2|\}$. A single $45^\circ$ rotation aligns both diagonal boundaries with the coordinate axes. Away from the boundaries, a depth-two tree with four leaves and cuts at zero represents the function; a finite cutpoint grid may introduce additional approximation error. Figure~\ref{main:fig:motivation}(b) shows a Gaussian spike,
$f(x)=0.02+0.65\exp\{-r^2/(2\cdot0.4^2)\}+0.70\exp\{-r^2/(2\cdot0.08^2)\}$,
where $r^2=x_1^2+x_2^2$.

\begin{figure*}[t]
\centering
\includegraphics[width=0.9\textwidth]{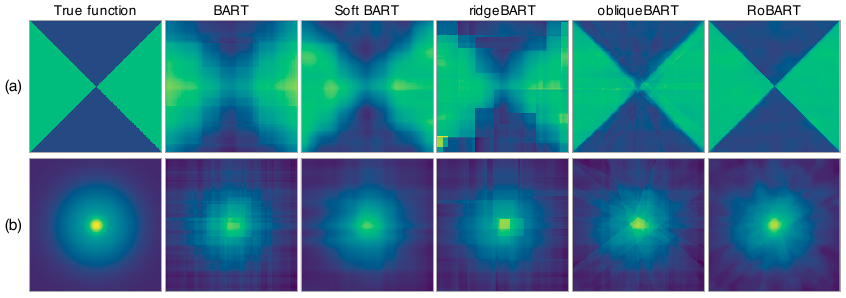}
\caption{True functions and posterior means for (a) the X-shaped signal and (b) a Gaussian spike. Both use $n=800$ uniform inputs on $[-1,1]^2$ and independent $N(0,0.1^2)$ errors. All methods use 200 trees. Color scales are shared within rows but differ between rows.}
\label{main:fig:motivation}
\end{figure*}

We propose RoBART. For $p$ predictors, each tree uses
a $p\times p$ rotation shared by all internal nodes, axis-aligned
splits in the rotated coordinates, and constant leaves. Rotations
may differ across trees, and each tree may split on only some
rotated coordinates.

For $B\in\R^{d\times p}$, $1\le d\le p$, $h(Bx)$ is unchanged by
$B\mapsto AB$ and $h(u)\mapsto h(A^{-1}u)$ for invertible
$A\in\R^{d\times d}$. Orthonormal rows permit orthogonal
basis changes. With known $d\ll p$, a row-orthonormal $d\times p$ matrix
uses fewer parameters, but learning $d$ requires changes to the % transformed
dimension, split coordinates, and cutpoint arrays
\citep{Green1995,CollinsFrancomRumsey2024}. Our $p\times p$
parameterization keeps inputs in $\R^p$ and split coordinates in
$\{1,\ldots,p\}$ without resolving nonidentifiability.

With tree structure, splitting coordinates, and cutpoint indices fixed, small rotations can preserve leaf assignments and admissible index sets, leaving the conditional posterior unchanged after integrating out the leaf means. Since the cutpoint grid depends on the rotated predictors, we jointly propose a Givens rotation sequence and affected cutpoints on the rebuilt grid.

We establish reversibility of this joint update and posterior contraction for additive regression functions with component-specific rotations and anisotropic smoothness. The contraction result assumes fixed predictor dimension, tree count, and component count. Theorem~\ref{main:thm:main} gives the rate under the conditions in Section~\ref{main:sec:theory}.

Theorem~\ref{main:thm:bart-lower} also constructs a fixed
H\"older ridge function on a tensor design for which the
axis-aligned BART posterior, under the same remaining prior
specifications, cannot attain the one-dimensional contraction
rate achieved by RoBART.

\section{RELATED WORK}
\label{main:sec:related}

Standard BART uses axis-aligned splits and constant leaf values \citep{ChipmanGeorgeMcCulloch2010}. Soft BART (SBART) replaces hard leaf indicators with smooth routing weights \citep{LineroYang2018}. ridgeBART models leaf outputs as linear combinations of ridge functions \citep{YeeGhoshDeshpande2024}. Our model retains hard splits and constant leaves and learns the coordinates in which each tree splits.

Oblique BART assigns a hyperplane to each split without requiring the normal vectors to come from one orthogonal matrix \citep{NguyenYeeDeshpande2025}. We restrict the normal vectors within a tree to rows of one rotation matrix. Random rotation ensembles also use separate rotations for base learners, but choose them before fitting \citep{BlaserFryzlewicz2016}. Here each rotation is learned jointly with its tree.

Bayesian Projection Pursuit Regression represents the mean as a sum of univariate ridge functions and learns their number by reversible-jump MCMC \citep{CollinsFrancomRumsey2024}. Each of our trees can instead use several rotated coordinates and represent interactions among them.

Our analysis builds on posterior contraction results for Bayesian trees and forests \citep{RockovaVanderPas2020,RockovaSaha2019} and anisotropic BART \citep{JeongRockova2023}, including the additive extension in Section~6.4 and Theorem~7 of the latter work. We account for a separate unknown rotation and the resulting cutpoint grid for each component. Theorem~\ref{main:thm:main} gives the rate, while
Theorem~\ref{main:thm:bart-lower} establishes a fixed-truth rate
separation from axis-aligned BART; proofs are in
Supplements~\ref{B:root} and~\ref{C:root}.

Figure~\ref{main:fig:motivation} compares the five BART variants
on two illustrative functions.

\section{MODEL}
\label{main:sec:model}

\begin{figure*}[t]
\centering
\includegraphics[width=\textwidth]{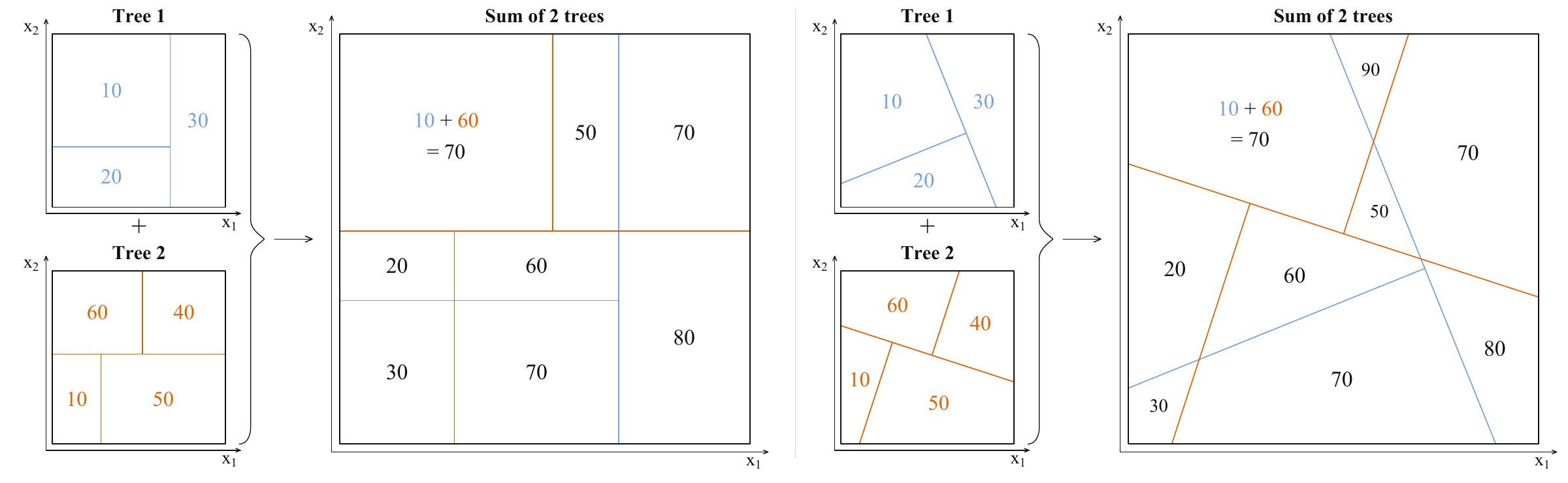}
\caption{Two-tree illustrations of BART (left) and RoBART (right). Within each half, the two small panels show individual tree contributions, and the larger panel shows their sum. Numbers are illustrative predictions in the original predictor coordinates. Blue and orange identify the two tree partitions.}
\label{main:fig:construction}
\end{figure*}

\subsection{Tree-specific rotations}

Fix $p\ge2$, a finite number of trees $T\ge1$, and the training design
$X=(x_1,\ldots,x_n)^{\trans}$, with $x_i\in\R^p$. We model
\begin{align*}
 Y_i=f(x_i)+\varepsilon_i,\quad
 \varepsilon_i\stackrel{\mathrm{iid}}{\sim}N(0,\sigma^2).
\end{align*}
For tree $t$, let $Q_t$ be a rotation matrix in
$\SO(p)=\{Q\in\R^{p\times p}:Q^{\trans}Q=I_p,\ \det Q=1\}$.
Let $\cT_t$ denote its rooted ordered binary topology, with internal nodes
$\mathcal I(\cT_t)$ and terminal nodes $\mathcal L(\cT_t)$.
An internal node $\eta$ splits on coordinate $v_\eta\in\{1,\ldots,p\}$ at
cutpoint index $c_\eta$; terminal node $\ell$ has mean $\mu_{t\ell}$.
Write $V_t=(v_\eta)$, $C_t=(c_\eta)$, and $M_t=(\mu_{t\ell})$ for these arrays.
The~regression mean is
\begin{align*}
 f(x)=\sum_{t=1}^{T}g_t(Q_tx;\cT_t,V_t,C_t,M_t),
 \quad Q_t\in\SO(p).
\end{align*}
Here $g_t$ is the regression tree specified by $(\cT_t,V_t,C_t,M_t)$.
Let $\xi_{t,j,k}(Q_t)$ be the candidate cutpoint with index $k$ for rotated
coordinate $j$, defined in \eqref{main:eq:grid}.
At an internal node $\eta$, an input is sent
to the left child exactly when
\begin{align*}
 (Q_tx)_{v_\eta}<\xi_{t,v_\eta,c_\eta}(Q_t),
\end{align*}
and to the right child otherwise. The output $g_t$ is the mean assigned to the
unique terminal node reached by these rules. Thus the split is axis-aligned
in the rotated coordinates. Let $e_j$ be the $j$th standard basis vector.
In the original predictor space, the split boundary has normal vector
$Q_t^{\trans}e_{v_\eta}$, the transpose of row $v_\eta$ of $Q_t$.
Splits on the same rotated coordinate have parallel boundaries.

Write $z_t(x)=Q_tx$ and
$f_t(x)=g_t(z_t(x);\cT_t,V_t,C_t,M_t)$, so that
$f(x)=\sum_{t=1}^{T}f_t(x)$.

Figure~\ref{main:fig:construction} illustrates $T=2$: standard BART uses $Q_1=Q_2=I_p$, whereas RoBART may use different rotations for the two trees.

The coordinates used by tree $t$ form the set
$\mathcal S_t=\{v_\eta:\eta\in\mathcal I(\cT_t)\}$.
The corresponding rows $Q_{t,\mathcal S_t,\cdot}$ are orthonormal.

\subsection{Cutpoint grids and prior}

For each rotated coordinate $j$, fix a positive integer $b_j$ as the number of candidate cutpoints throughout sampling. Define the training extrema by
\begin{align*}
 \underline z_{tj}(Q_t)=\min_{i\le n}(Q_tx_i)_j,
 \quad
 \overline z_{tj}(Q_t)=\max_{i\le n}(Q_tx_i)_j.
\end{align*}
For $k=0,\ldots,b_j-1$, set
\begin{align}
 \omega_{jk}&=\tfrac{k+1}{b_j+1},\notag\\[-0.3ex]
 \xi_{t,j,k}(Q_t)
 &=(1-\omega_{jk})\underline z_{tj}(Q_t)
 +\omega_{jk}\overline z_{tj}(Q_t).
 \label{main:eq:grid}
\end{align}
Write $\Xi_t(Q_t)=\{\xi_{t,j,k}(Q_t)\}$. These cutpoints are equally spaced
between the training minimum and maximum. A coordinate with zero range has
no available split. The grid is determined by $Q_t$ and the training predictors
and has no separate prior.

We generate each tree recursively. Write $\dep(\eta)$ for node depth, with root depth zero, and let $\rho_\ell\in[0,1]$ be the split-attempt probability at depth $\ell$. For a node $\eta$ and rotation $Q$, let $C_{\rm anc}$ denote its ancestor cutpoint indices. Ancestor splits on coordinate $j$ determine an interval; $\mathcal A_\eta(j;Q,C_{\rm anc})$ contains the grid indices with cutpoints strictly inside it and is empty for zero-range coordinates. We impose no minimum number of observations per leaf. The coordinates with an available cutpoint are
\begin{align*}
 \mathcal G_\eta(Q,C_{\rm anc})
 =\{j:\mathcal A_\eta(j;Q,C_{\rm anc})\ne\varnothing\}.
\end{align*}
At depth $\ell$, a node remains terminal with probability $1-\rho_\ell$. Otherwise, choose a coordinate uniformly from $\{1,\ldots,p\}$. If it has no admissible cutpoint, the node remains terminal. If it does, choose a cutpoint index uniformly from its admissible set and create two children. The probability that node $\eta$ remains terminal is therefore $1-\rho_{\dep(\eta)}\lvert\mathcal G_\eta\rvert/p$. For a valid state $\mathsf{s}=(Q,C)$, the tree prior is
\begin{equation}
\begin{aligned}
 &\pi_{\rm disc}(\cT,V,C\mid Q)
 ={}\prod\nolimits_{\eta\in\mathcal I(\cT)}
 \frac{\rho_{\dep(\eta)}}{p\,a_\eta(\mathsf{s})}\\[-0.2ex]
 &\qquad\times\prod\nolimits_{\ell\in\mathcal L(\cT)}
 \left\{1-\rho_{\dep(\ell)}\frac{g_\ell(\mathsf{s})}{p}\right\},
\end{aligned}
 \label{main:eq:discrete-prior}
\end{equation}
where $a_\eta(\mathsf{s})=\lvert\mathcal A_\eta(v_\eta;Q,C_{\rm anc})\rvert$
and $g_\eta(\mathsf{s})=\lvert\mathcal G_\eta(Q,C_{\rm anc})\rvert$;
invalid states have prior mass zero. Normalization follows from the recursive
construction: a path cannot reuse a coordinate's cutpoint index and therefore
has depth at most $\sum_jb_j$. Draw the rotation matrices independently from normalized
Haar measure on $\SO(p)$ and generate the trees independently conditional on
these matrices. Conditional on the forest structure, all leaf
means are independent $N(0,\tau^2)$ variables with fixed $0<\tau^2<\infty$.
Independently, $\sigma^2\sim\IG(a_\sigma,b_\sigma)$, with
$a_\sigma,b_\sigma>0$, using density proportional to
$(\sigma^2)^{-a_\sigma-1}\exp\{-b_\sigma/\sigma^2\}$.

The joint update assumes $0<\rho_\ell<1$. For the contraction analysis, we use $\rho_\ell=\nu^{\ell+1}$ with fixed $0<\nu<1$.

\subsection{Conditional posterior}

When updating tree $t$, form the partial residuals
\begin{align*}
 r_i^{(t)}=Y_i-\sum_{s\ne t}g_s(Q_sx_i).
\end{align*}
Suppress $t$ and condition on $(\cT,V,\sigma,\tau)$, the training design, and
the complete states of the other trees. Integrate out the current tree's leaf
means. For a terminal node $B$, let $I_B(\mathsf{s})$ be the indices of
observations assigned to $B$, and set
\begin{align*}
 b_B=\frac{\lvert I_B(\mathsf{s})\rvert}{\sigma^2},\qquad
 m_B=\sum_{i\in I_B(\mathsf{s})}\frac{r_i^{(t)}}{\sigma^2}.
\end{align*}
After integrating out the leaf means, the log marginal likelihood is $\ell_H(\mathsf{s})$ up to a constant independent of $\mathsf{s}$, where
\begin{align*}
 \ell_B(\mathsf{s})
 &=-\tfrac12\log(1+\tau^2b_B)
 +\frac{\tau^2m_B^2}{2(1+\tau^2b_B)},\\
 \ell_H(\mathsf{s})&=\sum\nolimits_{B\in\mathcal L(\cT)}\ell_B(\mathsf{s}).
\end{align*}
An empty leaf contributes $\ell_B=0$.
For fixed $(\cT,V)$, the space of cutpoint arrays is
\begin{align*}
 \mathcal C_{\max}(\cT,V)
 =\prod\nolimits_{\eta\in\mathcal I(\cT)}
 \{0,\ldots,b_{v_\eta}-1\}.
\end{align*}
With respect to normalized Haar measure on $\SO(p)$ and counting measure on
$\mathcal C_{\max}(\cT,V)$, the conditional posterior of $(Q,C)$ is
\begin{equation}
 \pi_H(\mathsf{s}\mid\mathrm{rest})
 \propto
 \exp\{\ell_H(\mathsf{s})\}\,
 \pi_{\rm disc}(\cT,V,C\mid Q),
 \label{main:eq:hardtarget}
\end{equation}
with density zero for invalid arrays. Here $\mathrm{rest}$ denotes the conditioning quantities listed above. The prior factor is \eqref{main:eq:discrete-prior} evaluated at the fixed $(\cT,V)$. Normalize the product jointly over $Q$ and $C$, without separately normalizing the tree prior over $C$ for each $Q$.

\section{POSTERIOR COMPUTATION}
\label{main:sec:update}

\subsection{Selecting coordinate pairs}

Write $[p]=\{1,\ldots,p\}$ and let
$\mathcal P=\{(a,b)\in[p]^2:a\ne b\}$. Let
\begin{align*}
 \mathcal S(\cT,V)=\{v_\eta:\eta\in\mathcal I(\cT)\},
 \qquad s=\lvert\mathcal S(\cT,V)\rvert,
\end{align*}
be the set of rotated-coordinate indices used by the current tree. This is $\mathcal S_t$ with the tree index omitted, not a subset of the original predictors. The set remains fixed during the joint update, although the corresponding rows of $Q$ may change. Conditional on $(\cT,V)$, draw ordered pairs $P_k=(a_k,b_k)\in\mathcal P$ independently with replacement from $q_P(\cdot\mid\cT,V)$. Define
\begin{align*}
 \mathcal P_{\rm in}
 &=\{(a,b)\in\mathcal P:a,b\in\mathcal S\},\\
 \mathcal P_{\rm out}
 &=(\mathcal S\times\mathcal S^c)\cup(\mathcal S^c\times\mathcal S),
\end{align*}
where $\mathcal S^c=[p]\setminus\mathcal S$. Fix $\rho_{\rm out}\in(0,1)$. For $2\le s<p$,
\begin{equation}
q_P(P\mid\cT,V)=
\begin{cases}
\dfrac{1-\rho_{\rm out}}{s(s-1)}, & P\in\mathcal P_{\rm in},\\
\dfrac{\rho_{\rm out}}{2s(p-s)}, & P\in\mathcal P_{\rm out},\\
0, & \text{otherwise}.
\end{cases}
\label{main:eq:plane-law}
\end{equation}
If $s\in\{0,p\}$, use the uniform distribution on all $p(p-1)$ ordered distinct pairs. If $s=1$, use the uniform distribution on the $2(p-1)$ pairs with one coordinate in $\mathcal S$ and the other outside it. Rotating two used rows changes the basis but preserves their joint span. Rotating a used row with an unused row can change the span of the used rows, without changing $V$ or $\mathcal S$.  Because the pairs are independent draws from the same distribution, reversing their order leaves their joint probability unchanged.

\subsection{Rotation proposal}

For an ordered pair $P=(a,b)$, let $G_P(\theta)$ equal $I_p$ except on the rows and columns indexed by $(a,b)$, in that order, where the block is
\begin{align*}
 \begin{pmatrix}\cos\theta&-\sin\theta\\
                 \sin\theta& \cos\theta\end{pmatrix}.
\end{align*}
Thus $G_P(\theta)\in\SO(p)$ and $G_P(\theta)^{-1}=G_P(-\theta)$. Fix $K_{\rm rot}\in\{1,2,\ldots\}$, $\gamma>0$, and $0<\theta_{\max}<\pi$. Starting from $Q_0=Q$, independently draw, for $k=1,\ldots,K_{\rm rot}$,
\begin{equation}
 P_k\sim q_P(\cdot\mid\cT,V),\;
 \theta_k\stackrel{\mathrm{iid}}{\sim}
 \TN_{[-\theta_{\max},\theta_{\max}]}(0,\gamma^2),
 \label{main:eq:centered-angle}
\end{equation}
where $\TN$ denotes the $N(0,\gamma^2)$ distribution truncated to the
indicated interval. Set
$Q_k=G_{P_k}(\theta_k)Q_{k-1}$. The proposed rotation is
\begin{equation}
 Q'=G_{P_{K_{\rm rot}}}(\theta_{K_{\rm rot}})\cdots
 G_{P_1}(\theta_1)Q\in\SO(p).
 \label{main:eq:givens-path}
\end{equation}
Changing the order of Givens rotations can change their product. To recover $Q$ from $Q'$, apply the same ordered pairs in reverse order with negated angles.

The centered truncated-normal law in \eqref{main:eq:centered-angle}
has an even density $q_\theta$, so the forward and reverse angle
densities agree.

\subsection{Cutpoint proposal}

After proposing a rotation, we use the current split to choose a reference location on the new cutpoint grid. We propose a random perturbation around this location rather than fixing the cutpoint there.

Let $\Xi_0=\Xi(Q)$ and $\Xi_1=\Xi(Q')$ denote the current and proposed cutpoint grids. Define
\begin{align*}
 \mathcal J_{\rm touch}&=\bigcup\nolimits_{k=1}^{K_{\rm rot}}\{a_k,b_k\},\\
 \mathcal N_{\rm touch}&=\{\eta\in\mathcal I(\cT):
 v_\eta\in\mathcal J_{\rm touch}\}.
\end{align*}
The sets $\mathcal J_{\rm touch}$ and $\mathcal N_{\rm touch}$ contain the coordinates selected for rotation and the internal nodes that split on them. Only these grid columns need to be recomputed; the other rows of $Q$ and rotated predictor values are unchanged.

For each affected node $\eta$, let $\mathcal I_\eta^0$ be the indices of observations reaching it under the current tree $(Q,C)$, and let $n_{\eta,L}^0$ be the number sent to its left child. These quantities are computed once and kept fixed while proposing $C'$.

We visit internal nodes in a fixed order, starting at the root and processing each parent before its children. At each affected node, the cutpoints already assigned to its ancestors determine the admissible index interval $[\underline c_\eta,\overline c_\eta]$. Other nodes retain their indices. If this interval is empty, the proposal fails $(\partial)$ and $(Q,C)$ remains unchanged.

For $0<n_{\eta,L}^0<\lvert\mathcal I_\eta^0\rvert$, let $z'_{\eta,(k)}$ be the $k$th order statistic of $\{(Q'x_i)_{v_\eta}:i\in\mathcal I_\eta^0\}$. Set
\begin{equation}
 t_\eta^\star=\frac{z'_{\eta,(n_{\eta,L}^0)}+z'_{\eta,(n_{\eta,L}^0+1)}}2.
 \label{main:eq:rank-threshold}
\end{equation}
Let $c_\eta^\star$ index the admissible cutpoint on $\Xi_1$ closest to $t_\eta^\star$, breaking ties by the smallest index. This defines the proposal center; it does not guarantee that the observation partition is preserved. If the node contains no observations, or all its observations go to the same child under the current tree, the center in \eqref{main:eq:rank-threshold} is undefined. Keep its current cutpoint index if it is admissible under the proposed ancestors; otherwise, return $\partial$ and retain the original state $(Q,C)$.

Fix $0<\rho_{\rm cut}<\infty$. With $W_\eta=\overline c_\eta-\underline c_\eta+1$ and $\kappa_\eta=\max\{1,\rho_{\rm cut}W_\eta\}$, define
\begin{align*}
 Z_\eta=\sum\nolimits_{u=\underline c_\eta}^{\overline c_\eta}
 \exp\{-(u-c_\eta^\star)^2/(2\kappa_\eta^2)\}.
\end{align*}
For an affected node with $0<n_{\eta,L}^0<\lvert\mathcal I_\eta^0\rvert$, draw the index from
\begin{align*}
 q_\eta(c)=Z_\eta^{-1}
 \exp\bigl\{-\tfrac{(c-c_\eta^\star)^2}{2\kappa_\eta^2}\bigr\}
 \ind\{\underline c_\eta\le c\le\overline c_\eta\}.
\end{align*}
For a generated array $C'$, let $q_C^F$ be the product of the conditional proposal probabilities at successive nodes. A node whose index is fixed by the rule above contributes a factor of one. Failed proposals leave $(Q,C)$ unchanged without resampling, so $q_C^F$ is not conditioned on success.

To compute $q_C^R$, evaluate the probability of proposing $C$ from $(Q',C')$ with proposed rotation $Q$. Recompute and fix the observation sets and left-child counts under $(Q',C')$. In the same node order, use $Qx_i$, $\Xi_0$, and ancestor indices from $C$ to evaluate the probability of each $c_\eta$ instead of drawing it. Their product is $q_C^R$. If the rule retains $c'_\eta$, only $c_\eta=c'_\eta$ has positive probability. Any zero factor gives $q_C^R=0$ and rejects the joint proposal.

\subsection{Acceptance probability and reversibility}

Write $\mathsf{s}=(Q,C)$ and $\mathsf{s}'=(Q',C')$, and define
\begin{align*}
 \Delta_{\rm disc}
 =\log\frac{\pi_{\rm disc}(\cT,V,C'\mid Q')}
 {\pi_{\rm disc}(\cT,V,C\mid Q)}.
\end{align*}
The pair probabilities and angle densities cancel between the proposed Givens sequence and its reverse. The grid is determined by the rotation, so the remaining proposal ratio is $q_C^R/q_C^F$. For $q_C^F>0$ and $q_C^R>0$, the log acceptance ratio is
\begin{equation}
 \log A=\ell_H(\mathsf{s}')-\ell_H(\mathsf{s})
 +\Delta_{\rm disc}+\log q_C^R-\log q_C^F.
 \label{main:eq:acceptance}
\end{equation}
The split-attempt and coordinate-selection factors cancel because $(\cT,V)$ is fixed. Thus $\Delta_{\rm disc}$ contains the changes in terminal-node probabilities and cutpoint prior probabilities. Accept with probability $1\wedge\exp(\log A)$. A zero proposed target density, a zero reverse probability, or a failed cutpoint proposal leaves $(Q,C)$ unchanged.

\begin{proposition}[Reversibility of the joint update]
\label{main:prop:reversibility}
Fix the conditioning quantities in Section~\ref{main:sec:model} and a
feasible $(\cT,V)$ for which \eqref{main:eq:hardtarget} has a positive
finite normalizing constant. Assume that the cutpoint grid is
evaluated exactly and that its map and tie rules are deterministic
and measurable. Under \eqref{main:eq:plane-law} and
\eqref{main:eq:centered-angle}, Algorithm~\ref{main:alg:update}, using the same
sequential cutpoint rule in both directions, including deterministic
holds and failure, is reversible with respect to \eqref{main:eq:hardtarget}.
\end{proposition}

Supplement~\ref{A:root} gives the cutpoint recursion, proves reversibility and
invariance of the completed tree update, and analyzes the cost of one
joint proposal. Reversibility establishes posterior invariance, not a
convergence-rate bound.

\begin{algorithm}[t]
\caption{Joint rotation and cutpoint update for tree $t$}
\label{main:alg:update}
\begin{algorithmic}[1]
\State Remove tree $t$ from the fitted sum and form $r^{(t)}$
\State Save $\mathsf{s}=(Q,C)$, $\Xi_0$, the rotated design, observation assignments, $\ell_H(\mathsf{s})$, and $\log\pi_{\rm disc}(\mathsf{s})$
\State Initialize the retained state to all saved quantities
\State Draw $P_{1:K_{\rm rot}}$ and $\theta_{1:K_{\rm rot}}$; construct $Q'$ by \eqref{main:eq:givens-path}
\State Recompute grid $\Xi_1$ for coordinates in $\mathcal J_{\rm touch}$
\State Propose $C'$ starting at the root, processing each node before its children; accumulate $\log q_C^F$
\If{the cutpoint proposal does not return $\partial$ and the candidate is valid}
 \State Compute proposed observation assignments, $\ell_H(\mathsf{s}')$, and $\log\pi_{\rm disc}(\mathsf{s}')$
 \State Evaluate $\log q_C^R$ by the same recursion from $\mathsf{s}'$ to $\mathsf{s}$, including the rules that retain an index
 \If{$q_C^R>0$ and the proposed target density is positive}
  \State With probability $1\wedge\exp(\log A)$ from \eqref{main:eq:acceptance}, replace the retained state by the proposed state and its associated quantities
 \EndIf
\EndIf
\end{algorithmic}
\end{algorithm}

\subsection{Completing the tree update}

At the retained rotation and grid, hold $r^{(t)}$, $\sigma$, $\tau$,
and the other trees fixed and apply birth--death kernels preserving
the leaf-collapsed conditional posterior. Birth proposals may sample
uniformly among coordinates with available cutpoints, but their
actual forward and reverse probabilities and the prior
\eqref{main:eq:discrete-prior} must enter the acceptance ratio. Without
using old leaf means, complete these updates by drawing all leaf
means from their Gaussian full conditional and restoring the tree
contribution, even after rotation rejection or failure.
Supplement~\ref{A:root} proves joint conditional invariance of this completed
update; a sweep through all trees followed by valid updates of the
remaining parameters preserves the full posterior.

\section{Posterior contraction}
\label{main:sec:theory}

\subsection{Fixed-design model and additive truth}

For asymptotics, make the sample-size dependence explicit:
\begin{equation}
 Y_{i,n}=f_0(x_{i,n})+\varepsilon_{i,n},\quad
 \varepsilon_{i,n}\stackrel{\mathrm{iid}}{\sim}N(0,\sigma_0^2),
 \label{main:eq:true-model}
\end{equation}
where $\sigma_0>0$ and $i=1,\ldots,n$. Define
\begin{align*}
 \begin{aligned}
 \norm{g}_n^2&=\tfrac1n\sum\nolimits_{i=1}^{n}g(x_{i,n})^2,\\
 Y^{(n)}&=(Y_{1,n},\ldots,Y_{n,n})^{\trans}.
 \end{aligned}
\end{align*}
Let $\E_0$ denote expectation under \eqref{main:eq:true-model}. Let $\Pi_n(\cdot\mid Y^{(n)})$ denote the posterior induced by the prior below, with $b_n$ candidate cutpoints on each rotated coordinate. The sequence $b_n\ge1$ consists of deterministic integers.

\Needspace{5\baselineskip}
\begin{assumption}[Fixed design and additive truth]
\label{main:ass:truth}
The ambient dimension $p\ge2$, number of trees $T$, and number of
components $R$ are fixed, with $1\le R\le T$. The fixed designs are
contained in one compact set $\cX\subset\R^p$ satisfying
\begin{align*}
 x_{i,n}\in\cX,\qquad
 \sup_{x\in\cX}\norm{x}_2\le M_X<\infty.
\end{align*}
\begingroup\displaywidowpenalty=10000
The truth has at least one representation on $\cX$ of the form
\begin{align*}
 f_0(x)=\sum_{r=1}^{R}h_{0r}\{\pi_{S_r}(Q_{0r}x)\},
 \qquad R\le T,
\end{align*}
\endgroup
where $Q_{0r}\in\SO(p)$ and
$S_r=\{s_{r1}<\cdots<s_{rd_r}\}\subseteq\{1,\ldots,p\}$, with
$1\le d_r\le p$. The projection is
$\pi_{S_r}(u)=(u_{s_{r1}},\ldots,u_{s_{rd_r}})^{\trans}$.
For fixed $\lambda_r>0$ and $\alpha_{rk}\in(0,1]$, each
$h_{0r}:[-M_X,M_X]^{d_r}\to\R$ satisfies
$\norm{h_{0r}}_\infty\le C_h<\infty$ and, for all $u,v$ in its domain,
\begin{align*}
 \abs{h_{0r}(u)-h_{0r}(v)}
 \le\lambda_r\sum_{k=1}^{d_r}\abs{u_k-v_k}^{\alpha_{rk}}.
\end{align*}
The dimensions $d_r$, radii $\lambda_r$, true rotation matrices and functions,
and smoothness vectors are fixed as $n\to\infty$.
\end{assumption}

Here $S_r$ selects coordinates of $Q_{0r}x$, and $\alpha_{rk}$ describes smoothness in coordinate $s_{rk}$. Since $Q_{0r}$ may differ across components, the same index may represent different directions. The dimension $d_r$ refers to the chosen representation; minimality and identifiability are not assumed.

Let
\begin{align*}
 \bar\alpha_r^{-1}=d_r^{-1}\sum\nolimits_{k=1}^{d_r}\alpha_{rk}^{-1}.
\end{align*}
Constants below may depend on the fixed truth parameters, $p,R,T$,
$M_X,C_h,r_{\min}$, and fixed prior hyperparameters, but not on $n$ or
on a rotation matrix in a stated neighborhood.

\subsection{Prior and grid conditions}

Use the prior in Section~\ref{main:sec:model} with $b_j=b_n$ for all $j$ and $\rho_\ell=\nu^{\ell+1}$ for fixed $0<\nu<1$. Keep all remaining prior hyperparameters fixed.

For $j\in S_r$, define the range of rotated coordinate $j$ by
\begin{align*}
 R_{n,rj}(Q)=\max_{i\le n}e_j^{\trans}Qx_{i,n}
 -\min_{i\le n}e_j^{\trans}Qx_{i,n}.
\end{align*}
For component $r$, define
\begin{align*}
 \widetilde K_{n,r}
 =\left\{\frac{n(\lambda_rd_r)^2}{\log n}\right\}^{d_r/(2\bar\alpha_r+d_r)},
\end{align*}
and choose a leaf count $K_{n,r}=2^{L_{n,r}}$ with $L_{n,r}\in\mathbb N$ and $K_{n,r}\asymp\widetilde K_{n,r}$.

\begin{assumption}[Active-range and grid conditions]
\label{main:ass:design-grid}
There exist $r_{\min}>0$, $C_b<\infty$, and $n_0$ such that
\begin{align}
\textnormal{(D1)}\quad
&\inf_{n\ge n_0}\min_{r\le R}\min_{j\in S_r}
 R_{n,rj}(Q_{0r})\ge r_{\min},
 \notag\\
\textnormal{(D2)}\quad
&b_n+1\ge\max_{r\le R}K_{n,r},
 \quad \log b_n\le C_b\log n.
 \label{main:eq:D2}
\end{align}
Here D2 is required for all $n\ge n_0$, after enlarging $n_0$ if necessary.
\end{assumption}
Condition D1 requires a positive lower bound on the ranges of the coordinates used by the truth; it does not require full rank or a space-filling design. The choice $b_n=n$ satisfies D2 for all sufficiently large $n$, since $d_r/(2\bar\alpha_r+d_r)<1$, and requires no knowledge of the smoothness or coordinate sets.

\subsection{Contraction rate}

Define
\begin{align*}
 \eps_{n,r}&=(\lambda_rd_r)^{d_r/(2\bar\alpha_r+d_r)}
 \left(\frac{\log n}{n}\right)^{\bar\alpha_r/(2\bar\alpha_r+d_r)},
 \\
 \eps_n&=\sum_{r=1}^{R}\eps_{n,r}.
\end{align*}

\Needspace{8\baselineskip}
\begin{theorem}[Posterior contraction]
\label{main:thm:main}
Under the fixed-design model \eqref{main:eq:true-model}, Assumptions~\ref{main:ass:truth} and~\ref{main:ass:design-grid}, and the preceding prior specification, there exists a sufficiently large constant $M<\infty$, independent of $n$, such that
\begin{align*}
 \E_0\Pi_n\bigl[
 \norm{f-f_0}_n+\abs{\sigma-\sigma_0}>M\eps_n
 \mid Y^{(n)}\bigr]\to0.
\end{align*}
\end{theorem}

For fixed $R$, the slowest component determines the order of $\eps_n$.
Supplement~\ref{B:root} uses one rotated tree per component to verify
the local prior-mass, entropy, and sieve-tail conditions of
\citet{Jeong2025}. The result concerns $f$ and $\sigma$,
not the component representations.

\subsection{Comparison with axis-aligned BART}
\label{main:sec:bart-lower}

The preceding upper bound does not determine whether axis-aligned
BART can attain the same rate. Let $\Pi_n^{\B}$ denote the posterior
under the preceding prior specification with $Q_t=I_p$
deterministically for every tree. All remaining specifications,
including the hyperparameters, empirical-range grid rule, and
budget $b_n$, are unchanged.

\Needspace{8\baselineskip}
\begin{theorem}[Posterior contraction lower bound for BART]
\label{main:thm:bart-lower}
Fix integers $p\ge2$ and $T\ge1$, and constants
$0<\alpha<1$ and $\lambda,C_h,\sigma_0>0$. Let $\cX=[0,1]^p$.
For integers $m\ge1$ and $n=2^{mp}$, let the design enumerate
\[
\mathcal X_n
=\left\{(i+1/2)2^{-m}:i=0,\ldots,2^m-1\right\}^{p}.
\]
Assume \eqref{main:eq:D2} with $R=d_1=1$,
$\lambda_1=\lambda$, and $\alpha_{11}=\alpha$.

There exist $Q_0\in\SO(p)$ and
$h_0:[-\sqrt p,\sqrt p]\to\R$, both independent of $n$,
such that $\norm{h_0}_\infty\le C_h$ and
$\abs{h_0(u)-h_0(v)}\le\lambda\abs{u-v}^{\alpha}$
on this interval, and the truth
$f_0(x)=h_0\{\pi_{\{1\}}(Q_0x)\}$ on $[0,1]^p$
has the following properties. Set
\begin{align*}
\varepsilon_n
&=\lambda^{1/(2\alpha+1)}
  \left(\frac{\log n}{n}\right)^{\alpha/(2\alpha+1)},\\
\varepsilon_n^{\B}
&=\left(\frac{\log n}{n}\right)^{\alpha/(2\alpha+p)}.
\end{align*}
For some fixed $c>0$ and a sufficiently large fixed $M_0<\infty$,
\begin{align*}
\E_0\Pi_n^{\B}\bigl[
\norm{f-f_0}_n\le c\varepsilon_n^{\B}
\mid Y^{(n)}\bigr]
&\rightarrow0,
\\
\E_0\Pi_n\bigl[
\norm{f-f_0}_n+\abs{\sigma-\sigma_0}>M_0\varepsilon_n
\mid Y^{(n)}\bigr]
&\rightarrow0.
\end{align*}
All limits are taken as $m\to\infty$ along $n=2^{mp}$.
\end{theorem}

Since $\varepsilon_n=o(\varepsilon_n^{\B})$, the BART posterior
assigns vanishing mass to the ball of radius $M\varepsilon_n$
for every fixed $M>0$ at this truth. Thus the stated axis-aligned
prior need not adapt to a one-dimensional rotated representation.
The proof combines an approximation lower bound for the entire
forest with posterior control of its total leaf count.
Supplement~\ref{C:root} gives the construction and proof;
the RoBART bound follows from Theorem~\ref{main:thm:main}.

\section{SIMULATION STUDY}
\label{main:sec:simulation}

\begin{figure*}[!t]
\centering
\includegraphics[width=0.95\textwidth]{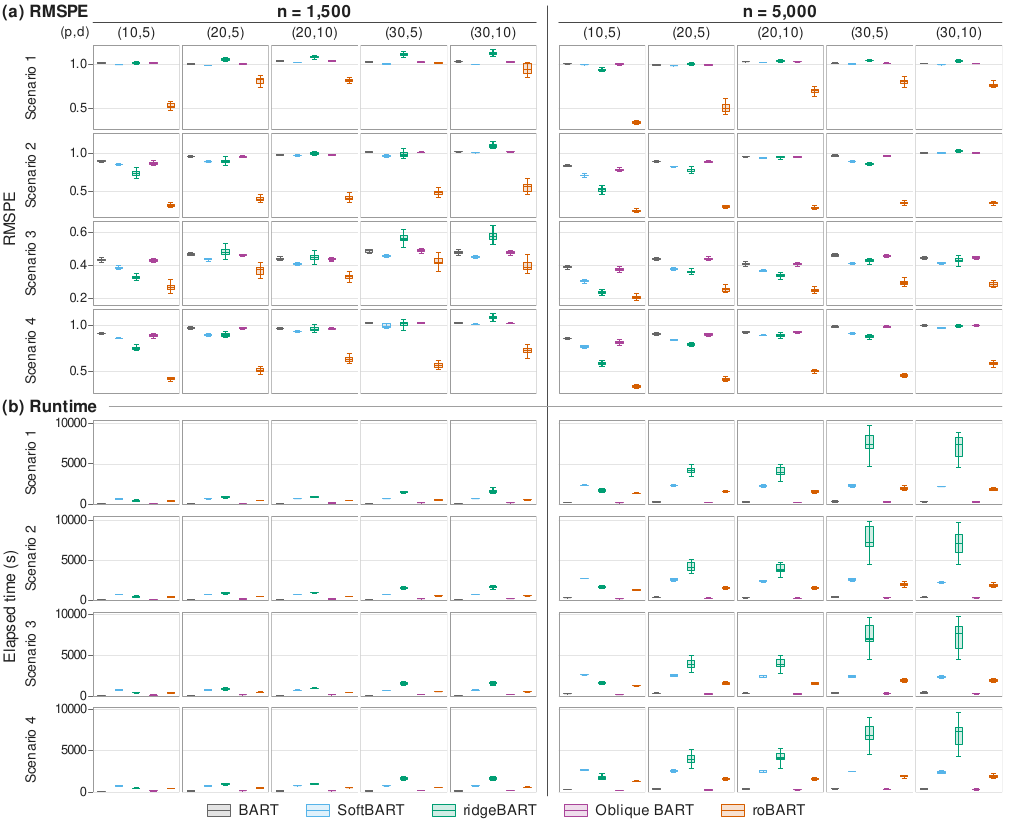}
\caption{RMSPE (a) and elapsed time in seconds (b) at $n\in\{1{,}500,5{,}000\}$.
Rows give Scenarios~1--4 and columns give $(p,d)$. Boxes summarize thirty
replications; RMSPE is measured against the true regression function.}
\label{main:fig:rmspe}
\end{figure*}

\subsection{Benchmark design}
\label{main:sec:benchmark}

We compare five methods on four regression functions of $d$ nonorthogonal
linear combinations of $p$ observed predictors. We use
$(p,d)=(10,5)$, $(20,5)$, $(20,10)$, $(30,5)$, and $(30,10)$,
training sizes $n\in\{1{,}500,5{,}000\}$, and $N=2{,}000$ test inputs.

We generate $X_i\sim N_p(0,I_p)$ and $Y_i=f_s(X_i)+\varepsilon_i$,
with independent $N(0,0.5^2)$ errors. For each raw function $r_s$ below,
set $f_s=(r_s-\mu_s)/\sqrt{v_s}$, where $\mu_s$ and $v_s$ are its
sample mean and variance on $100{,}000$ independent calibration inputs.
These constants remain fixed across sample sizes and replications.

Let $U=[u_1,\ldots,u_d]\in\R^{p\times d}$ have orthonormal columns,
with $u_1=p^{-1/2}\mathbf 1_p$. The remaining columns are obtained by
orthogonalizing Gaussian vectors against $u_1$ and each other.
All four scenarios use $z=Ax$, where
$A=(0.25I_d+0.75\mathbf 1_d\mathbf 1_d^{\trans})^{1/2}U^{\trans}$
and the square root is symmetric positive definite.
Geometrically, the transformation turns projected spherical Gaussian
contours into ellipsoids elongated along $\mathbf 1_d$.
The coordinates of $AX$ have unit variance and pairwise correlation $0.75$. Write $t=\Phi(z)$, with the standard normal
distribution function applied coordinatewise.

\noindent\textbf{Scenario 1: Sum of sines.}
The raw function is $r_1(x)=d^{-1/2}\sum_{j=1}^{d}\sin(3.3z_j)$.

\noindent\textbf{Scenario 2: Smooth surface.}
Define $H_d(t)=\sin[(10/\sqrt d)\{\sum_{j=1}^{d}(t_j-1/2)^2-d/12\}]$
\citep{JeongRockova2023} and set $r_2(x)=H_d(t)$.

\noindent\textbf{Scenario 3: Friedman blocks.}
Let $F(t_1,\ldots,t_5)=10\sin(\pi t_1t_2)+20(t_3-1/2)^2+10t_4+5t_5$.
For $d=5$, set $r_3(x)=F(t_1,\ldots,t_5)$; for $d=10$, set
$r_3(x)=\{F(t_1,\ldots,t_5)+F(t_6,\ldots,t_{10})\}/\sqrt2$.

\noindent\textbf{Scenario 4: Smooth surface with jumps.}
Let $S(t)=\sum_{j=1}^{d}(t_j-1/2)$ and
$W(t)=\sum_{j=1}^{d}(-1)^j(t_j-1/2)$. Define
$J_d(t)=\ind\{S(t)\le0,\ W(t)>0\}+\ind\{S(t)>0,\ W(t)\le0\}$
and set $r_4(x)=H_d(t)+J_d(t)$.
The jump boundaries are hyperplanes in $t$-space but need not be
hyperplanes in the observed predictor space.

For each $(p,d)$, scenarios and thirty replications share a fixed
transformation matrix and training/test inputs; only noise and MCMC
seeds vary. Methods use identical datasets, and scenarios share noise
within each replication. The smaller training sample contains the
first $1{,}500$ observations of the larger one, including their noise,
and both sizes use the same test inputs.

All methods fit $200$ trees to the observed data and average
$1{,}000$ unthinned draws after $9{,}000$ burn-in iterations for
prediction. RoBART uses $K_{\rm rot}=5$ and $\gamma=0.08$. See Appendix~\ref{D:root}.

\subsection{Prediction accuracy and runtime}
\label{main:sec:results}

Figure~\ref{main:fig:rmspe}(a) shows lower median RMSPE for RoBART across the
displayed settings in Scenarios~2--4. The improvement is less uniform
in Scenario~1, where errors remain close to one in some $p=30$ settings
at $n=1{,}500$.

Figure~\ref{main:fig:rmspe}(b) compares elapsed times under the reported
fitting settings. RoBART requires more time than BART and Oblique BART,
but less than Soft BART and, in most displayed settings, ridgeBART.

\section{DISCUSSION}\label{main:sec:discussion}

RoBART combines tree-specific rotations with a joint rotation and cutpoint update that accounts for the changing numerical grid. We establish reversibility of the joint conditional update and posterior contraction for additive regression functions with component-specific rotations.
For a fixed H\"older ridge function on a tensor design, the
axis-aligned counterpart need not attain RoBART's
one-dimensional contraction rate.

The contraction result assumes fixed $p,T,R$ with $R\le T$. Extensions to growing $p$ and convergence analysis of the sampler require further study.

% End-of-main marker is outside the research text and used for page-limit checking.
\label{main:end}

% Continue in two columns so acknowledgments do not occupy an isolated page.
\section*{Acknowledgments}
\phantomsection\addcontentsline{toc}{section}{Acknowledgments}
This research was supported by the Basic Science Research Program
through the National Research Foundation of Korea (NRF), funded by
the Ministry of Education (RS-2025-25426611).

% Preserve the union of references printed in the original main paper and B.
\nocite{ghosal2000,ghosal2007,ghosal2017,RockovaSaha2019}
\bibliographystyle{apalike}
\bibliography{reference}

% One integrated document, not a concatenation of PDFs.
\clearpage
\onecolumn
\appendix
\setcounter{section}{0}
\setcounter{equation}{0}
\setcounter{algorithm}{0}
\setcounter{table}{0}
\makeatletter
\@addtoreset{equation}{section}
\@addtoreset{algorithm}{section}
\@addtoreset{table}{section}
\makeatother
\renewcommand{\theequation}{\thesection.\arabic{equation}}
\renewcommand{\thealgorithm}{\thesection.\arabic{algorithm}}
\renewcommand{\thetable}{\thesection.\arabic{table}}
\renewcommand{\theHsection}{app.\Alph{section}}
\renewcommand{\theHsubsection}{\theHsection.\arabic{subsection}}
\renewcommand{\theHsubsubsection}{\theHsubsection.\arabic{subsubsection}}
\renewcommand{\theHequation}{\theHsection.\arabic{equation}}
\renewcommand{\theHtable}{\theHsection.\arabic{table}}
\providecommand{\theHalgorithm}{}
\renewcommand{\theHalgorithm}{\theHsection.\arabic{algorithm}}
\phantomsection
\pdfbookmark[0]{Appendices}{appendices}
\begin{center}{\Large\bfseries Appendices}\end{center}

\section{Reversibility of the Joint Rotation and Cutpoint Update}\label{A:root}
\begingroup
\renewcommand{\abs}[1]{\left\lvert#1\right\rvert}
\subsection{Conditional target and proposal path}\label{A:section:1}
\label{A:sec:A-target}

Fix $p\ge2$, the training design, and positive integer budgets
$b_1,\ldots,b_p$. For one tree, condition on its ordered binary topology
$\cT$, split-label array $V$, partial residual vector $r^{(t)}$,
$\sigma,\tau>0$, and the complete states of the other trees.
Suppress the tree index. The joint update acts on $\mathsf s=(Q,C)$
in $\mathsf Z=\SO(p)\times\cC_{\max}$, where
$\cC_{\max}=\prod_{\eta\in\cI(\cT)}\{0,\ldots,b_{v_\eta}-1\}$.
Let $\lambda$ be normalized Haar measure times counting measure on this
space. With $\ell_H$ and the full tree prior $\pi_{\rm disc}$ as defined
\Mainlink{eq:hardtarget}{in the main paper}, the collapsed conditional target is
\begin{align*}
\widetilde\pi_H(\mathsf s)
 =e^{\ell_H(\mathsf s)}\pi_{\rm disc}(\cT,V,C\mid Q),
\qquad
\Pi_H(d\mathsf s)=Z^{-1}\widetilde\pi_H(\mathsf s)\lambda(d\mathsf s).
\end{align*}
Invalid arrays have density zero. As in the \Mainlink{prop:reversibility}{main proposition}, assume
$0<Z=\int\widetilde\pi_H\,d\lambda<\infty$.
Normalization is joint over $Q$ and $C$: the tree-prior factor is
not separately normalized over $C$ for each $Q$.

The grid $\Xi(Q)$ is the main paper's coordinatewise empirical-range grid,
using only the full training design. Its evaluation is exact, with
deterministic measurable tie rules as in the \Mainlink{prop:reversibility}{main proposition}. For a node $\eta$ and label $j$,
$\cA_\eta(j;Q,C_{\rm anc})$ is the set of grid indices strictly inside
the interval imposed by its same-label ancestors; it is empty for a
zero-range coordinate. There is no minimum leaf occupancy, and split
equality goes to the right child. Write
$a_\eta(\mathsf s)=|\cA_\eta(v_\eta;Q,C_{\rm anc})|$ and
$g_\eta(\mathsf s)=\#\{j:\cA_\eta(j;Q,C_{\rm anc})\ne\varnothing\}$.
The depth probabilities satisfy $0<\rho_\ell<1$; no asymptotic design or
grid condition is needed for this conditional update.

Let $\mathcal P=\{(a,b)\in[p]^2:a\ne b\}$ and
$K=K_{\rm rot}\ge1$. Conditional on the fixed $(\cT,V)$, draw
$P_1,\ldots,P_K$ independently from the pair law $q_P$ \Mainlink{eq:plane-law}{in the main paper}.
Independently draw $\theta_1,\ldots,\theta_K$ from its centered
$N(0,\gamma^2)$ law truncated to $[-\theta_{\max},\theta_{\max}]$,
where $\gamma>0$ and $0<\theta_{\max}<\pi$. Its density $q_\theta$ is even.
All tuning parameters and activation schedules are fixed in advance.
Write $H(P,\theta)=G_{P_K}(\theta_K)\cdots G_{P_1}(\theta_1)$, so
$Q'=H(P,\theta)Q$. The reversed path
$(P_K,-\theta_K),\ldots,(P_1,-\theta_1)$ reconstructs $Q$ because
$G_P(\theta)^{-1}=G_P(-\theta)$. The coordinates within each ordered pair
are not exchanged. Independence under the same fixed pair law and
evenness of $q_\theta$ make the forward and reversed path probabilities
equal. We keep the realized path as an auxiliary variable; an endpoint
proposal density with respect to Haar measure is not assumed.

\subsection{Sequential cutpoint proposal}\label{A:section:2}
\label{A:sec:A-cutpoints}

For the realized path, set
$\mathcal J_{\rm touch}=\bigcup_{k=1}^K\{a_k,b_k\}$ and
$\mathcal N_{\rm touch}=\{\eta\in\cI(\cT):v_\eta\in\mathcal J_{\rm touch}\}$.
Only those grid columns change under the coordinatewise grid rule.
Fix one parent-before-child order on $\cI(\cT)$ and $\rho_{\rm cut}>0$.
Under a departure state $\mathsf s^0=(Q^0,C^0)$, compute the observation
indices $\mathcal I_\eta^0$ reaching each node and its left-child count
$n_{\eta,L}^0$. Keep both fixed throughout the recursion.
Working ancestor indices determine the admissible target interval;
they do not change these reference observations.

At a touched node with
$0<n_{\eta,L}^0<|\mathcal I_\eta^0|$, let $z'_{\eta,(k)}$ be the
order statistics of $\{(Q'x_i)_{v_\eta}:i\in\mathcal I_\eta^0\}$.
For the current admissible interval
$\cA_\eta^F=\{\underline c_\eta,\ldots,\overline c_\eta\}$,
set $t_\eta^\star=(z'_{\eta,(n_{\eta,L}^0)}
+z'_{\eta,(n_{\eta,L}^0+1)})/2$ and choose
$c_\eta^\star=\min\arg\min_{c\in\cA_\eta^F}
|\xi_{v_\eta,c}(Q')-t_\eta^\star|$.
With $W_\eta=|\cA_\eta^F|$ and
$\kappa_\eta=\max\{1,\rho_{\rm cut}W_\eta\}$, use
\begin{equation}
q_\eta(c)=
\frac{\exp\{-(c-c_\eta^\star)^2/(2\kappa_\eta^2)\}}
 {\sum_{u\in\cA_\eta^F}\exp\{-(u-c_\eta^\star)^2/(2\kappa_\eta^2)\}}
\ind\{c\in\cA_\eta^F\}.
\label{A:eq:A-cut-law}
\end{equation}
This center need not preserve observation counts or partitions after
snapping to the grid or changing ancestors. At an untouched node, or
when all reference observations go to one child or the node is empty,
retain $c_\eta^0$ if admissible; otherwise return failure $\partial$.
An empty admissible set also gives $\partial$.

\begin{algorithm}[t]
\caption{Sequential cutpoint proposal or mass evaluation}
\label{A:alg:A-cutpoints}
\begin{algorithmic}[1]
\Require Departure state $(Q^0,C^0)$, target $Q'$ and $\Xi(Q')$,
$\mathcal N_{\rm touch}$; generation or evaluation mode, with requested $D$ in evaluation mode
\State Compute and freeze $\mathcal I_\eta^0,n_{\eta,L}^0$ under $(Q^0,C^0)$
\State Initialize $\bar C\gets C^0$ and $\log q\gets0$
\For{internal nodes $\eta\in\cI(\cT)$ in the fixed parent-before-child order}
  \State $\cA\gets\cA_\eta(v_\eta;Q',\bar C_{\rm anc})$
  \If{$\cA=\varnothing$}
    \State \Return $\partial$ in generation mode, or $(D,-\infty)$ in evaluation mode
  \EndIf
  \If{$\eta\notin\mathcal N_{\rm touch}$ or $n_{\eta,L}^0\in\{0,|\mathcal I_\eta^0|\}$}
    \If{$c_\eta^0\notin\cA$}
      \State \Return $\partial$ in generation mode, or $(D,-\infty)$ in evaluation mode
    \EndIf
    \State $h_\eta(c)\gets\ind\{c=c_\eta^0\}$
  \Else
    \State Compute $t_\eta^\star,c_\eta^\star,\kappa_\eta$ and set $h_\eta$ by \eqref{A:eq:A-cut-law}
  \EndIf
  \If{generation mode}
    \State Draw $d_\eta\sim h_\eta$
  \Else
    \State $d_\eta\gets D_\eta$
    \If{$h_\eta(d_\eta)=0$}
      \State \Return $(D,-\infty)$
    \EndIf
  \EndIf
  \State $\log q\gets\log q+\log h_\eta(d_\eta)$; commit $\bar c_\eta\gets d_\eta$
\EndFor
\State \Return $(\bar C,\log q)$
\end{algorithmic}
\end{algorithm}

\Needspace{10\baselineskip}
Algorithm~\ref{A:alg:A-cutpoints} defines a measurable law on
$\cC_{\max}\cup\{\partial\}$: paths, order statistics, and nearest-index
choices are measurable finite operations with deterministic ties.
Each local law is a normalized point mass or a finite distribution
with a positive normalizer. Sequential composition assigns each
completed array its product of local masses and assigns failed prefixes
to $\partial$. Thus completed-array and failure probabilities sum to one.
For a completed array $C'$, call its product probability $q_C^F$.
It is not conditioned on success. Failure leaves the entire $(Q,C)$
state unchanged, without resampling or committing a partial array.
For a valid departure state, an untouched index remains admissible:
its frame row and all same-label ancestor indices are unchanged.

To obtain $q_C^R$, run the evaluation mode with departure $(Q',C')$,
target rotation $Q$, and requested array $D=C$. Recompute and freeze
reference observations and left counts at $(Q',C')$; initialize the
working array at $C'$ and commit each requested old index before visiting
descendants. Recompute all reverse centers, scales, and normalizers.
A reverse deterministic hold assigns positive mass only when
$c_\eta=c'_\eta$. Reverse evaluation uses the same node order, not its
reversal. The reversed Givens path has the same touched set.

\subsection{Reversibility of the joint update}\label{A:section:3}
\label{A:sec:A-reversibility}

\begin{proof}[Proof of the reversibility proposition in the main paper]
Use $\mathsf U=\mathcal P^K\times[-\theta_{\max},\theta_{\max}]^K
\times\cC_{\max}$ for non-failure auxiliary outputs. Its reference
measure $\kappa$ is counting measure
on pairs and cutpoint arrays times Lebesgue measure on angles.
For $\mathsf s=(Q,C)$ and $u=(P_{1:K},\theta_{1:K},C')$, set
$\mathsf s'=(H(P,\theta)Q,C')$ and
$u^\dagger=(P_{K:1},-\theta_{K:1},C)$.
The map $\mathfrak J(\mathsf s,u)=(\mathsf s',u^\dagger)$ is a
measurable involution on the full product space, including tuples with
zero target or proposal density: applying the reversed rotations,
array exchange, and signed path reversal twice restores every coordinate.

For each fixed pair--angle path, $Q\mapsto H(P,\theta)Q$ preserves
Haar measure. Reversing the pair sequence and exchanging $C,C'$ preserve
counting measure, and the signed permutation of angles preserves
Lebesgue measure. Applying the Haar change of variables first for each
path and then these permutations shows by Tonelli's theorem that
$\int h\circ\mathfrak J\,d(\lambda\otimes\kappa)
=\int h\,d(\lambda\otimes\kappa)$ for every nonnegative measurable $h$.
This accounts for the dependence of $H$ on the angles. The grid is a
function of $Q$, not an extra coordinate of the reference measure.

The non-failure auxiliary density is
$q_{\mathsf s}(u)=\{\prod_{k=1}^Kq_P(P_k\mid\cT,V)q_\theta(\theta_k)\}q_C^F$.
Give inadmissible outputs density zero, and extend this density by zero
on zero-target departure states. Its integral is at most one; the missing
mass is failure probability, not a normalizing factor.
Define $F(\mathsf s,u)=\widetilde\pi_H(\mathsf s)q_{\mathsf s}(u)$.
Since $\mathfrak J$ is an involution,
$(F\circ\mathfrak J)(\mathsf s,u)
=\widetilde\pi_H(\mathsf s')q_{\mathsf s'}(u^\dagger)$.
At positive-density endpoints with positive forward and reverse mass,
the pair and angle factors cancel, giving
\begin{align*}
\log\frac{F\circ\mathfrak J}{F}
=\ell_H(\mathsf s')-\ell_H(\mathsf s)
 +\Delta_{\rm disc}+\log q_C^R-\log q_C^F.
\end{align*}
Only the cutpoint and terminal-node prior factors remain in
$\Delta_{\rm disc}$, because $(\cT,V)$ is fixed:
\begin{align*}
\Delta_{\rm disc}
=\sum_{\eta\in\cI(\cT)}\log\frac{a_\eta(\mathsf s)}{a_\eta(\mathsf s')}
 +\sum_{\ell\in\cL(\cT)}
 \log\frac{1-\rho_{\dep(\ell)}g_\ell(\mathsf s')/p}
              {1-\rho_{\dep(\ell)}g_\ell(\mathsf s)/p}.
\end{align*}
For $F>0$, let $\alpha=1\wedge(F\circ\mathfrak J)/F$; set
$\alpha=0$ when $F=0$. A zero proposed target or reverse mass therefore
rejects the proposal without evaluating an undefined log ratio.
This is the acceptance rule of the main paper, and
\begin{align*}
F\alpha=\min\{F,F\circ\mathfrak J\}
       =(F\circ\mathfrak J)(\alpha\circ\mathfrak J).
\end{align*}
For measurable $B_0,B_1\subseteq\mathsf Z$, the target-weighted accepted
transition mass is
\begin{align*}
\frac1Z\int
 \ind\{\mathsf s\in B_0\}\ind\{\mathsf s'\in B_1\}
 \min\{F,F\circ\mathfrak J\}\,
 \lambda(d\mathsf s)\kappa(du).
\end{align*}
The measure-preserving involution exchanges the indicators and fixes the
minimum, so this expression is symmetric in $B_0,B_1$.
The remaining probability is
$r(\mathsf s)=1-\int q_{\mathsf s}(u)\alpha(\mathsf s,u)\,\kappa(du)
\in[0,1]$. Rejections and failures thus contribute the symmetric
mass $\int_{B_0\cap B_1}r(\mathsf s)\Pi_H(d\mathsf s)$.
Together these prove detailed balance for the full joint-update kernel.
\end{proof}

The proof uses the exact grid and proposal laws. A subsequent numerical
projection of $Q'$ without accounting for it defines a different kernel.
The result establishes invariance, not irreducibility or a mixing-rate bound.

\subsection{Completing the tree update}\label{A:section:4}
\label{A:sec:A-completion}

For tree $t$, write $\zeta_t=(Q_t,C_t,\cT_t,V_t)$ and let $M_t$
be its leaf vector. Conditional on $r^{(t)}$, $\sigma,\tau$, and the other
trees, write the joint posterior as
$\Pi_t(d\zeta_t,dM_t)=\overline\Pi_t(d\zeta_t)\Pi_t(dM_t\mid\zeta_t)$.

\begin{Acorollary}[Invariance of the completed tree update]
\label{A:cor:A-completion}
Let $L_t$ be a finite composition of structural kernels preserving
$\overline\Pi_t$, with the above conditioning quantities fixed and without
using old leaf values. Following $L_t$ with a draw of all leaf means
from $\Pi_t(dM_t'\mid\zeta_t')$ preserves $\Pi_t$.
\end{Acorollary}
\begin{proof}
Since $\overline\Pi_tL_t=\overline\Pi_t$, the target-weighted completed
transition satisfies
\begin{align*}
\begin{aligned}
&\int\Pi_t(d\zeta_t,dM_t)L_t(\zeta_t,d\zeta_t')
                         \Pi_t(dM_t'\mid\zeta_t')\\
&\qquad=\int\overline\Pi_t(d\zeta_t)L_t(\zeta_t,d\zeta_t')
                         \Pi_t(dM_t'\mid\zeta_t')\\
&\qquad=\overline\Pi_t(d\zeta_t')\Pi_t(dM_t'\mid\zeta_t')
       =\Pi_t(d\zeta_t',dM_t').
\end{aligned}
\end{align*}
An independent draw from $\overline\Pi_t$ is not required.
\end{proof}

\Needspace{8\baselineskip}
The joint kernel leaves $(\cT,V)$ fixed and preserves their conditional
target, so it preserves $\overline\Pi_t$. Composing it with birth--death
kernels for the same collapsed target gives the $L_t$ used above.
Their actual proposal probabilities and the full discrete prior must enter their
acceptance ratios, including when a birth proposal samples only among
available labels. Refresh all leaf means after the final structural
state, even after rotation rejection or failure, before restoring the
tree contribution or using its leaf values in another block.
Completed tree kernels in systematic backfitting order, followed by valid
updates of the remaining model parameters, preserve the full posterior.
The sweep need not itself be reversible, and no finite-burn-in guarantee
follows from this invariance statement.

\subsection{Computational cost}\label{A:section:5}
\label{A:sec:A-cost}

We count arithmetic and comparison operations for one joint proposal,
including reverse probability evaluation, with elementary functions and
scalar random variates treated as unit-cost operations. The current
rotated training design is cached, its proposed version uses a full
scratch copy, unchanged grid columns are shared, and only touched frame
rows are buffered. Birth--death moves and the subsequent leaf refresh
are not included in this bound.

Let $D$ be the maximum leaf depth, $N$ the number of nodes,
$J=|\mathcal J_{\rm touch}|$, $I=|\mathcal N_{\rm touch}|$, and
$b=\max_jb_j$. For departure or proposed routing $a\in\{0,1\}$,
let $d_i^a$ be the depth reached by observation $i$. Then
\begin{align*}
\sum_{\eta\in\cI(\cT)}|\mathcal I_\eta^a|
  =\sum_{i=1}^n d_i^a\le nD.
\end{align*}
Following each observation's path collects the touched-node values in
both directions within $2nD$ incidences. Only two adjacent order statistics
are needed per nondegenerate node. Worst-case linear-time selection
therefore gives the costs in Table~\ref{A:tab:A-cost}; metadata and
point-mass checks add $O(N)$.

\begin{table}[!ht]
\centering
\caption{Operation bounds for one joint proposal under the stated storage convention.}
\label{A:tab:A-cost}
\begin{tabular}{@{}p{0.67\textwidth}p{0.27\textwidth}@{}}
\toprule
Operation & Cost\\
\midrule
Copy the rotated training design & $O(np)$\\
Apply the Givens path to design columns and frame rows & $O\{K_{\rm rot}(n+p)\}$\\
Recompute touched extrema and materialize their grid columns & $O\{J(n+b)\}$\\
Route observations, collect reference values, select order statistics,
and evaluate collapsed likelihoods & $O\{n(1+D)+N\}$\\
Enumerate local cutpoint normalizers and evaluate or draw indices & $O(Ib)$\\
Evaluate discrete-prior factors by inspecting labels and ancestors & $O\{pN(1+D)\}$\\
\bottomrule
\end{tabular}
\end{table}

Direct enumeration costs $O(W_\eta)\le O(b)$ per touched node and direction;
nearest-index search is included. With cached range flags and index bounds,
the prior calculation inspects at most $p$ labels and $D$ ancestors per node. Touched-row storage costs $O(Jp)$, absorbed
by $J\le2K_{\rm rot}$. Combining the terms gives
\begin{equation}
O\{np+K_{\rm rot}(n+p)+J(n+b)+n(1+D)+Ib+pN(1+D)\}.
\label{A:eq:A-total-cost}
\end{equation}
Sorting every reference set instead of linear-time selection replaces
the corresponding $n(1+D)$ term by the conservative bound
$n(1+D)\log(n+1)$. An average-linear selection routine gives only its
corresponding average guarantee. Early rejection can shorten an attempt;
failed attempts are not retried until success.

With all reference sets materialized, one tree's resident-plus-working
storage is $O\{np+p^2+\sum_jb_j+n(1+D)+N+K_{\rm rot}\}$ under these
conventions. For fixed $p,K_{\rm rot},D,b$, $N\le2^{D+1}-1$ makes
\eqref{A:eq:A-total-cost} an $O(n)$ bound. Here $D$ is the realized depth,
not a prior depth cap. This is a per-attempt bound, not a uniform
linear-time result for increasing budgets, depths, complete sweeps,
or a prescribed Monte Carlo accuracy.

\endgroup

\clearpage
\section{Posterior Contraction for \texorpdfstring{\NoCaseChange{RoBART}}{RoBART}}\label{B:root}
\begingroup
\renewcommand{\R}{\mathbb{R}}
\renewcommand{\E}{\mathbb{E}}
\renewcommand{\abs}[1]{\left\lvert#1\right\rvert}
\renewcommand{\inner}[2]{\left\langle #1, #2\right\rangle}
\subsection{Introduction}\label{B:section:1}

We prove the \Mainlink{thm:main}{posterior contraction theorem in the main paper} under the prior and design conditions below. Following the additive construction of \cite{JeongRockova2023}, we assign one approximating tree to each component, allowing separate rotations and rotation-dependent cutpoint grids, and take the remaining trees to be nearly zero. We verify local prior mass, sieve entropy, and prior tail bounds, then apply the Gaussian contraction criterion of \cite{Jeong2025}. Theorem~\ref{B:thm:main} restates the \Mainlink{thm:main}{main paper's result}; numbered references prefixed by B refer to this supplement.

\subsection{Model, prior, and assumptions}\label{B:section:2}

\subsubsection{Fixed-design regression model}\label{B:subsection:2.1}

We observe
\begin{equation}
Y_{i,n} = f_0(x_{i,n}) + \varepsilon_{i,n},\qquad \varepsilon_{i,n} \stackrel{\mathrm{iid}}{\sim} N(0,\sigma_0^2),\qquad i=1,\dots,n,
\label{B:eq:model}
\end{equation}
with fixed $\sigma_0>0$ and fixed design points $x_{i,n}\in\R^p$.  The ambient dimension $p\ge2$ is fixed.  For a measurable function $g$ on the design space, define the empirical norm
\[
\norm{g}_n^2 = \frac{1}{n}\sum_{i=1}^n |g(x_{i,n})|^2.
\]
We consider a sequence of fixed designs contained in one bounded set: there exists $M_X<\infty$ such that
\begin{equation}
\sup_{n\ge1}\max_{1\le i\le n} \norm{x_{i,n}}_2 \le M_X.
\label{B:eq:bounded_design}
\end{equation}
Take a compact design set $\mathcal X$ contained in the Euclidean ball of radius $M_X$ and containing every $x_{i,n}$. Write $Y^{(n)}=(Y_{1,n},\ldots,Y_{n,n})^\top$. Let $P_{f,\sigma}^{(n)}=\bigotimes_{i=1}^n N\{f(x_{i,n}),\sigma^2\}$, set $P_0^{(n)}=P_{f_0,\sigma_0}^{(n)}$, and denote expectation under $P_0^{(n)}$ by $\E_0$.

Write $\mathbb P_n$ for the prior on the forest parameters and the noise scale, and $\Pi_n$ for its induced prior on $(f,\sigma)$. The corresponding posterior is $\Pi_n(\cdot\mid Y^{(n)})$. The prior induced on $(\mu_f,\sigma)$, where $\mu_f=(f(x_{1,n}),\ldots,f(x_{n,n}))^\top$, is denoted by $\overline\Pi_n$.

Define $O(p)=\{Q:Q^\top Q=I_p\}$ and $\SO(p)=\{Q\in O(p):\det Q=1\}$. Every $Q\in\SO(p)$ preserves Euclidean norm, so $Qx_{i,n}\in\DX=[-M_X,M_X]^p$.

\subsubsection{Additive rotated truth and smoothness classes}\label{B:subsection:2.2}

Fix integers $1\le R\le T$, both independent of $n$. On $\mathcal X$, the true regression function is assumed to satisfy
\begin{equation}
f_0(x)=\sum_{r=1}^R h_{0r}(\pi_{S_r}(Q_{0r}x)),
\qquad Q_{0r}\in \SO(p),
\qquad S_r\subset\{1,\dots,p\},
\qquad |S_r|=d_r,
\label{B:eq:truth}
\end{equation}
For each component let $1\le d_r\le p$, write $S_r=\{s_{r1}<\cdots<s_{rd_r}\}$, and define $\pi_{S_r}(u)=(u_{s_{r1}},\ldots,u_{s_{rd_r}})^\top$. The set $S_r$ indexes coordinates in the $r$th rotated system, not necessarily in the original system.

For any orthogonal matrix $Q\in \SO(p)$, we use the composition notation
\[
h_{0r}(\pi_{S_r}(Q\cdot)): x\mapsto h_{0r}(\pi_{S_r}(Qx)).
\]
Thus $h_{0r}(\pi_{S_r}(Q_{0r}\cdot))$ is the $r$th true component, while $h_{0r}(\pi_{S_r}(Q\cdot))$ for $Q\ne Q_{0r}$ means that the same canonical anisotropic template is evaluated after a different rotation.

\Needspace{15\baselineskip}
\begin{Bdefinition}[Componentwise sparse anisotropic H\"older classes]
\label{B:def:holder}
For $r=1,\ldots,R$, let $\alpha_r=(\alpha_{r1},\ldots,\alpha_{rd_r})$ with $\alpha_{rk}\in(0,1]$ and let $\lambda_r>0$.  We write
\[
h_{0r}\in \mathcal H^{\alpha_r,d_r}_{\lambda_r}
\]
if
\[
|h_{0r}(u)-h_{0r}(v)|\le \lambda_r\sum_{k=1}^{d_r}|u_k-v_k|^{\alpha_{rk}}
\]
for all $u,v\in[-M_X,M_X]^{d_r}$. Each $h_{0r}$ is defined on this box and satisfies $\norm{h_{0r}}_\infty\le C_h$, with the supremum taken over the entire box.  The harmonic-mean and minimal smoothness parameters are
\begin{align*}
\bar\alpha_r^{-1}=\frac1{d_r}\sum_{k=1}^{d_r}\alpha_{rk}^{-1},
\qquad
\alpha_{r,\min}=\min_{1\le k\le d_r}\alpha_{rk}.
\end{align*}
The exponent $\alpha_{rk}$ corresponds to ambient row $s_{rk}$. All truth parameters are fixed as $n\to\infty$. Constants may depend on these quantities, $p,R,T$, and fixed prior hyperparameters, but not on $n$ or a rotation in a stated neighborhood.
\end{Bdefinition}

\subsubsection{RoBART prior with uniform coordinate selection}\label{B:subsection:2.3}

The prior is a tree-specific rotation forest.  For $t=1,\dots,T$,
\begin{align*}
f(x)=\sum_{t=1}^T g_t(Q_tx),\qquad Q_t\in \SO(p).
\end{align*}
For tree $t$, let $\cT_t$ be its rooted ordered binary topology, $V_t$ its split-label array, $C_t$ its array of zero-based grid indices, and $M_t=(\mu_{t\ell})_{\ell\in\mathcal L(\cT_t)}$ its leaf-mean vector. Write $K_t=|\mathcal L(\cT_t)|$. As in the main paper, $g_t(Q_tx)$ abbreviates $g_t(Q_tx;\cT_t,V_t,C_t,M_t)$. The complete discrete state is $(\cT_t,V_t,C_t)$, not the topology alone. The number of trees $T$ is fixed.

We impose the following prior structure.
\begin{enumerate}[label=(P\arabic*),leftmargin=2.8em]
\item The rotations $Q_1,\ldots,Q_T$ are independent draws from
normalized Haar measure on $\SO(p)$.

\item At each split attempt, a coordinate is drawn uniformly from
$\{1,\ldots,p\}$, so every coordinate has selection probability $1/p$.
These draws use fresh randomness independently across nodes and trees
and independently of the rotations and split-attempt draws.

\item Conditional on $Q_1,\ldots,Q_T$, the discrete trees
$(\cT_t,V_t,C_t)$, $t=1,\ldots,T$, are generated independently.
At each node of depth $\ell$, a split is attempted with probability
$\nu^{\ell+1}$, where $\nu\in(0,1)$, using independent Bernoulli draws.
If no split is attempted, the node is terminal. Otherwise, a coordinate
is drawn according to \textnormal{(P2)}. If that coordinate has no
admissible cutpoint in the current box, the node is terminal; otherwise,
its cutpoint is drawn uniformly from the admissible elements of the
corresponding rotated split-net, using a fresh draw conditional on the
current box.

\item Conditional on
$Q_1,\ldots,Q_T,(\cT_t,V_t,C_t)_{t=1}^T$,
all leaf means are mutually independent $N(0,\tau^2)$ variables,
where $0<\tau^2<\infty$ is fixed.

\item Independently of the forest, $\sigma^2\sim\operatorname{IG}(a_\sigma,b_\sigma)$
with fixed $a_\sigma,b_\sigma>0$ and density proportional to
$(\sigma^2)^{-a_\sigma-1}\exp\{-b_\sigma/\sigma^2\}$.
\end{enumerate}
For a fixed valid state $(\cT_t,V_t,C_t)$ and rotation $Q_t$, let $a_\eta$ be the number of admissible indices for the prescribed label at internal node $\eta$, and let $g_\ell$ be the number of labels with an admissible index at terminal node $\ell$. These counts depend on $Q_t$ and the ancestor indices. Writing $\operatorname{dep}$ for node depth, the recursion gives $\pi_{\rm disc}(\cT_t,V_t,C_t\mid Q_t)=\prod_{\eta\in\mathcal I(\cT_t)}\{\nu^{\operatorname{dep}(\eta)+1}/(p a_\eta)\}\allowbreak\prod_{\ell\in\mathcal L(\cT_t)}\{1-\nu^{\operatorname{dep}(\ell)+1}g_\ell/p\}$. Invalid states have mass zero. Each terminal factor is at least $1-\nu$. The factor $1/p$ is used along the generating path, before conditioning on any subsequent descendant generation.

\subsubsection{Contraction rate and approximation scale}\label{B:subsection:2.4}

For $r=1,\ldots,R$, define
\begin{equation}
\varepsilon_{n,r}
=
(\lambda_r d_r)^{d_r/(2\bar\alpha_r+d_r)}
\left(\frac{\log n}{n}\right)^{\bar\alpha_r/(2\bar\alpha_r+d_r)},
\qquad
\epsn=\sum_{r=1}^R\varepsilon_{n,r}.
\label{B:eq:epsn}
\end{equation}
For each component choose an integer $L_{n,r}$ and a leaf count $K_{n,r}=2^{L_{n,r}}$ such that
\begin{equation}
K_{n,r}\asymp
\left(\frac{n(\lambda_rd_r)^2}{\log n}\right)^{d_r/(2\bar\alpha_r+d_r)}.
\label{B:eq:Kb}
\end{equation}
Then
\begin{equation}
K_{n,r}\log n\asymp n\varepsilon_{n,r}^2,
\qquad
\sum_{r=1}^RK_{n,r}\log n\lesssim n\epsn^2.
\label{B:eq:balance}
\end{equation}
We also define the shrinking active-row rotation radii
\begin{equation}
\delta_{n,r}:=c_{\delta,r}
\left(\frac{\varepsilon_{n,r}}{\lambda_rd_r}\right)^{1/\alpha_{r,\min}},
\label{B:eq:deltan}
\end{equation}
where $c_{\delta,r}>0$ is sufficiently small.  Since all component parameters are fixed,
\begin{align*}
\log(1/\epsn)\lesssim\log n,
\qquad
\log(1/\delta_{n,r})\lesssim\log n.
\end{align*}
If
\[
a_r=\frac{\bar\alpha_r}{2\bar\alpha_r+d_r},
\qquad
a_\star=\min_{1\le r\le R}a_r\in(0,1/2),
\]
then
\begin{equation}
\epsn\asymp\left(\frac{\log n}{n}\right)^{a_\star},
\qquad
\frac{n\epsn^2}{\log n}\to\infty.
\label{B:eq:rate_arithmetic}
\end{equation}
Moreover, for every fixed $\kappa_K\in(0,1-2a_\star)$,
\begin{equation}
\frac{n\epsn^2}{\log n}\ge n^{\kappa_K}
\label{B:eq:tree_growth}
\end{equation}
for all sufficiently large $n$.  This polynomial growth will be used in the sieve-tail bound.

\subsubsection{Rotation-adaptive split-nets and additive anisotropic k-d trees}\label{B:subsection:2.5}

\paragraph{Empirical intervals and candidate cutpoints.}
For each $n$, fix a positive integer budget $b_n$. At every $Q\in\SO(p)$, the prior uses the full empirical box of $Qx_{1,n},\ldots,Qx_{n,n}$ and the $b_n$ equally spaced interior cutpoints on each coordinate of positive range. A zero-range coordinate has no candidate cutpoints. The prior selects among all $p$ coordinates, and imposes no leaf-occupancy restriction. Only the approximating tree for component $r$ is restricted to $S_r$.

For component $r$, define
\[
\mathcal Q_{n,r}
=
\left\{Q\in\SO(p):\norm{Q_{S_r,\cdot}-Q_{0r,S_r,\cdot}}_F\le\delta_{n,r}\right\}.
\]
For any $Q\in\SO(p)$, write $u_i(Q)=Qx_{i,n}$ and, for $j\in S_r$, set
\[
m_{rj}(Q):=\min_{1\le i\le n}u_{ij}(Q),
\qquad
M_{rj}(Q):=\max_{1\le i\le n}u_{ij}(Q),
\qquad
R_{n,rj}(Q):=M_{rj}(Q)-m_{rj}(Q).
\]
The active empirical box is
\[
D_{Q,r}:=\prod_{j\in S_r}[m_{rj}(Q),M_{rj}(Q)].
\]
For $R_{n,rj}(Q)>0$, define the coordinatewise candidate set directly by
\[
[Z_{n,r}(Q)]_j
=
\left\{m_{rj}(Q)+\frac{k}{b_n+1}R_{n,rj}(Q):k=1,\ldots,b_n\right\}.
\]
If $R_{n,rj}(Q)=0$, set $[Z_{n,r}(Q)]_j=\varnothing$.
The family $Z_{n,r}(Q)=([Z_{n,r}(Q)]_j:j\in S_r)$ is the rotation-adaptive uniform split-net. Its entries are the corresponding candidate sets of the full prior grid; they depend only on $Q$ and the design.

\paragraph{Admissible splits and tree partitions.}
For a box $D=\prod_{j\in S}I_j$ indexed by an ordered coordinate set $S$, write $[D]_j=I_j$ and let $\operatorname{len}([D]_j)$ be its length. A split-net on $D$ is a family $Z=([Z]_j:j\in S)$ of finite candidate sets $[Z]_j\subset[D]_j$. For a sub-box $\Omega\subset D$, the admissible cutpoints along $j$ are
\[
\mathcal C_j(\Omega;Z):=[Z]_j\cap\operatorname{int}([\Omega]_j).
\]
For $\tau\in\mathcal C_j(\Omega;Z)$, split $\Omega$ into
\[
\Omega^- = \Omega\cap\{u\in D:u_j<\tau\},
\qquad
\Omega^+ = \Omega\cap\{u\in D:u_j\ge\tau\}.
\]
We use this left/right convention throughout, including for design points on a boundary. A $Z$-tree partition $\mathcal T=\{\Omega_1,\ldots,\Omega_K\}$ is obtained from $D$ by finitely many such binary splits. Its step-function class is
\[
\mathcal F_{\mathcal T}
=
\left\{u\mapsto\sum_{k=1}^K\beta_k\1(u\in\Omega_k):\beta_1,\ldots,\beta_K\in\R\right\}.
\]
For a bounded $g:D\to\R$, write $\norm{g}_{\infty,D}=\sup_{u\in D}|g(u)|$.

\paragraph{Midpoint split and anisotropic schedule.}
If $\mathcal C_j(\Omega;Z)=\{\tau_{(1)}<\cdots<\tau_{(m)}\}$ with $m\ge1$, its midpoint split uses $\tau_{(\lceil m/2\rceil)}$. If $m=0$, the split is unavailable.

\Needspace{10\baselineskip}
Let $S=\{s_1<\cdots<s_d\}$ and let $\alpha=(\alpha_1,\ldots,\alpha_d)$ with $\alpha_k\in(0,1]$, where $\alpha_k$ corresponds to $s_k$. For $L\in\mathbb N$, define $A_{kd,S}(D;Z,\alpha,L)$ as follows.
\begin{enumerate}[label=(\roman*),leftmargin=2.8em]
\item Start from $\{D\}$ and initialize $l_k=0$ for $k=1,\ldots,d$.
\item Choose $k^\star\in\arg\min_{1\le k\le d}\alpha_kl_k$, breaking ties by the smallest $k$.
\item If a midpoint split along coordinate $s_{k^\star}$ is unavailable in at least one current leaf, stop.
\item Otherwise, split every current leaf along $s_{k^\star}$ and set $l_{k^\star}\leftarrow l_{k^\star}+1$.
\item Repeat until $\sum_{k=1}^d l_k=L$ or a required split is unavailable.
\end{enumerate}
The output is a $Z$-tree partition with split counts $(l_1,\ldots,l_d)$ and the recorded split labels and global cutpoint ranks. Reaching depth $L$ gives $2^L$ leaves. Since $S$ is ordered, the tie rule is also the smallest-coordinate-label rule.

\paragraph{Discrete template and numerical partition.}
Apply the preceding schedule with $S_r,\alpha_r,L_{n,r}$ to the integer grid whose root interval has endpoints $0$ and $b_n+1$ on every active coordinate. Record the ordered topology, split labels, and global grid ranks; a global rank $k\in\{1,\ldots,b_n\}$ is stored as the main paper's index $c=k-1$. Denote this complete discrete state by $\mathfrak t_{n,r}$. It is determined without choosing a rotation. Whenever its prescribed splits are available under $Q$, let $\mathcal P_{n,r}(Q)$ be the numerical partition of $D_{Q,r}$ obtained by evaluating these indices on $Z_{n,r}(Q)$. Proposition~\ref{B:prop:dynamic_uniform_sparse} shows that the construction completes all $L_{n,r}$ rounds and is valid throughout $\mathcal Q_{n,r}$ for all sufficiently large $n$. Leaving coordinates outside $S_r$ unsplit gives the corresponding $p$-dimensional tree. Values of $h_{0r}$ use the ordered coordinates $s_{r1},\ldots,s_{rd_r}$.

\begin{Bassumption}[Additive rotation-adaptive uniform split-net regime]
\label{B:ass:dynamic_uniform}
The active split-net $Z_{n,r}(Q)$ is constructed as above for every $r\le R$ and every $Q\in\mathcal Q_{n,r}$.  We assume the following.
\begin{enumerate}[label=(D\arabic*),leftmargin=2.8em]
\item \textbf{Active-coordinate design nondegeneracy.}  There exist constants $r_{\min}>0$ and $n_0\in\mathbb N$ such that, for all $n\ge n_0$,
\[
\min_{1\le r\le R}\min_{j\in S_r}R_{n,rj}(Q_{0r})\ge r_{\min}.
\]
\item \textbf{Grid resolution.}  The sequence $b_n$ satisfies
\[
\max_{1\le r\le R}2^{L_{n,r}}\le b_n+1,
\qquad
\log b_n\le C_b\log n
\]
for some constant $C_b>0$ and all sufficiently large $n$.
\end{enumerate}
\end{Bassumption}

\begin{Bremark}[Active-coordinate ranges]
\label{B:rem:range}
Condition \textnormal{(D1)} keeps the empirical grid intervals in the truth-active directions nondegenerate under local rotations, as quantified in Proposition~\ref{B:prop:dynamic_uniform_sparse}(i). A sufficient condition is that, for all large $n$, the convex hull of the design contains a Euclidean ball of fixed radius $c>0$: every unit-direction range is then at least $2c$. This is stronger than \textnormal{(D1)}, which restricts only the finitely many truth-active directions. In particular, \textnormal{(D1)} does not require a full-rank empirical covariance, a design density, or a positive number of observations in every leaf; it may hold even for designs supported on a lower-dimensional affine subspace.
\end{Bremark}

\begin{Bremark}[Grid resolution]
\label{B:rem:grid}
The lower bound in \textnormal{(D2)} ensures that every prescribed midpoint split is available; the upper bound controls the cutpoint contribution to the negative log prior mass. Let $\kappa_\star=\max_{1\le r\le R}d_r/(2\bar\alpha_r+d_r)<1$. By \eqref{B:eq:Kb}, $b_n=\lceil n^\kappa\rceil$ with $\kappa>\kappa_\star$ satisfies \textnormal{(D2)}, so $b_n=n$ works without knowing the component smoothness. The resolution requirement is conservative: $b_n+1\ge\max_{r,k}2^{l_{rk}}$ suffices for the prescribed full-depth split counts, whereas \textnormal{(D2)} uses $2^{L_{n,r}}$. We retain the stated condition for simplicity. It requires an increasing cutpoint budget and does not cover a fixed finite grid.
\end{Bremark}

\Needspace{10\baselineskip}
\begin{Bproposition}[Properties of the additive rotation-adaptive split-nets]
\label{B:prop:dynamic_uniform_sparse}
Under Assumption~\ref{B:ass:dynamic_uniform}, for all sufficiently large $n$, every $r\le R$, and every $Q\in\mathcal Q_{n,r}$, the active split-net $Z_{n,r}(Q)$ satisfies the following properties.
\begin{enumerate}[label=(\roman*),leftmargin=2.8em]
\item \textbf{Local range stability and positivity.}  For every $j\in S_r$,
\[
|m_{rj}(Q)-m_{rj}(Q_{0r})|
\le M_X\norm{Q_{S_r,\cdot}-Q_{0r,S_r,\cdot}}_F,
\]
\[
|M_{rj}(Q)-M_{rj}(Q_{0r})|
\le M_X\norm{Q_{S_r,\cdot}-Q_{0r,S_r,\cdot}}_F,
\]
and
\[
|R_{n,rj}(Q)-R_{n,rj}(Q_{0r})|
\le2M_X\norm{Q_{S_r,\cdot}-Q_{0r,S_r,\cdot}}_F.
\]
In particular, $R_{n,rj}(Q)\ge R_{n,rj}(Q_{0r})/2>0$ for all sufficiently large $n$.
\item \textbf{Cutpoint count and side-length control.} Each active coordinate has $b_n$ candidate cutpoints, with $\log b_n\le C_b\log n$.
Moreover, $\mathfrak t_{n,r}$ is valid under $Q$ and reaches full depth $L_{n,r}$. If $(l_{rk}:1\le k\le d_r)$ denotes its split counts, then every leaf $\Omega\in\mathcal P_{n,r}(Q)$ satisfies
\[
\operatorname{len}([\Omega]_{s_{rk}})\le2R_{n,r,s_{rk}}(Q)2^{-l_{rk}},
\qquad k=1,\ldots,d_r.
\]
\item \textbf{Uniform sparse anisotropic approximation on $D_{Q,r}$.}  There exists a step function $\widehat h_{r,Q}\in\mathcal F_{\mathcal P_{n,r}(Q)}$ such that
\[
\norm{h_{0r}-\widehat h_{r,Q}}_{\infty,D_{Q,r}}
\le C_A\lambda_rd_r2^{-L_{n,r}\bar\alpha_r/d_r}
\lesssim\varepsilon_{n,r}.
\]
\item \textbf{Fixed template and coefficient centers.} The state $\mathfrak t_{n,r}$ is the same for all $Q\in\mathcal Q_{n,r}$; its numerical thresholds and partition may vary. Choose $\widehat h_{r,Q}$ by assigning each leaf the value of $h_{0r}$ at the midpoint of its closure, and write $\beta_{n,r}(Q)$ for these values in left-to-right leaf order. The map $Q\mapsto\beta_{n,r}(Q)$ is continuous on $\mathcal Q_{n,r}$, hence Borel measurable, and satisfies $\norm{\beta_{n,r}(Q)}_\infty\le C_h$ and $\norm{\beta_{n,r}(Q)}_2^2\le C_h^2 K_{n,r}$.
\end{enumerate}
\end{Bproposition}

\begin{proof}
Fix $r$ and $Q\in\mathcal Q_{n,r}$.

For (i), fix $j\in S_r$.  Since
\[
m_{rj}(Q)=\min_{1\le i\le n}e_j^\top Qx_{i,n},
\]
one has
\[
|m_{rj}(Q)-m_{rj}(Q_{0r})|
\le\max_{1\le i\le n}|e_j^\top(Q-Q_{0r})x_{i,n}|
\le M_X\norm{Q_{S_r,\cdot}-Q_{0r,S_r,\cdot}}_F.
\]
The same argument applies to $M_{rj}(Q)$, and the bound for $R_{n,rj}(Q)$ follows by the triangle inequality.  Since $\delta_{n,r}\to0$, condition \textnormal{(D1)} gives positivity uniformly over $Q\in\mathcal Q_{n,r}$ for sufficiently large $n$.

For (ii), part (i) makes the $b_n$ equally spaced candidates distinct on each active coordinate. Fix $k\in\{1,\ldots,d_r\}$ and put $j=s_{rk}$. The interval $[m_{rj}(Q),M_{rj}(Q)]$ is partitioned into $b_n+1$ elementary cells, each of length
\[
\Delta_{rj}(Q)=\frac{R_{n,rj}(Q)}{b_n+1}.
\]
Every subinterval generated by recursive midpoint splits is a union of consecutive elementary cells. An interval containing $N\ge2$ cells has $N-1$ admissible cutpoints; its midpoint split has rank $\lceil(N-1)/2\rceil=\lfloor N/2\rfloor$ and gives children with $\lfloor N/2\rfloor$ and $\lceil N/2\rceil$ cells. Let $N_j(\Omega)$ be the number of elementary cells in the $j$th coordinate interval of a leaf $\Omega$. Iterating these floor and ceiling bounds, after $l_{rk}$ midpoint splits along coordinate $j$,
\[
\left\lfloor\frac{b_n+1}{2^{l_{rk}}}\right\rfloor\le N_j(\Omega)\le\left\lceil\frac{b_n+1}{2^{l_{rk}}}\right\rceil.
\]
Before the $a$th split along coordinate $j$, every current interval contains at least
\[
\left\lfloor\frac{b_n+1}{2^{a-1}}\right\rfloor
\]
elementary cells.  Since $b_n+1\ge2^{L_{n,r}}$, this quantity is at least two for every $a\le L_{n,r}$.  Therefore every requested midpoint split is available and the algorithm reaches full depth $L_{n,r}$.  Finally,
\[
\operatorname{len}([\Omega]_j)
=N_j(\Omega)\Delta_{rj}(Q)
\le2R_{n,rj}(Q)2^{-l_{rk}},
\]
because $b_n+1\ge2^{L_{n,r}}\ge2^{l_{rk}}$.

For (iii), for each leaf $\Omega\in\mathcal P_{n,r}(Q)$ let $u_\Omega$ be the coordinatewise midpoint of its closure and define
\[
\widehat h_{r,Q}(u)
=
\sum_{\Omega\in\mathcal P_{n,r}(Q)}h_{0r}(u_\Omega)\1(u\in\Omega).
\]
The split counts satisfy
\[
\max_{1\le k\le d_r}\alpha_{rk}l_{rk}
-
\min_{1\le k\le d_r}\alpha_{rk}l_{rk}
\le1.
\]
Indeed, the claim holds at initialization, and each update increments a coordinate minimizing $\alpha_{rk}l_{rk}$ by at most one because $\alpha_{rk}\le1$.  If $m_r=\min_{1\le k\le d_r}\alpha_{rk}l_{rk}$, then
\[
L_{n,r}
=
\sum_{k=1}^{d_r}l_{rk}
\le(m_r+1)\sum_{k=1}^{d_r}\alpha_{rk}^{-1}
=(m_r+1)\frac{d_r}{\bar\alpha_r},
\]
so $m_r\ge L_{n,r}\bar\alpha_r/d_r-1$.  H\"older continuity and part (ii) now give
\[
|h_{0r}(u)-\widehat h_{r,Q}(u)|
\le C(M_X)\lambda_r\sum_{k=1}^{d_r}2^{-\alpha_{rk}l_{rk}}
\lesssim\lambda_rd_r2^{-L_{n,r}\bar\alpha_r/d_r}.
\]
Taking the supremum and using \eqref{B:eq:Kb} proves the claim.

For (iv), the selected coordinate at each round depends only on the scaled split counts and their deterministic tie rule. An index interval with integer endpoints $a<b$ has candidates $a+1,\ldots,b-1$; when it is split, its median rank is $a+\lfloor(b-a)/2\rfloor$. This recursion does not depend on $Q$, and part (ii) ensures that every required split exists. Thus the ordered topology, labels, and global indices in $\mathfrak t_{n,r}$ are fixed. Each numerical leaf endpoint is $m_{rj}(Q)+kR_{n,rj}(Q)/(b_n+1)$ for a fixed integer $k\in\{0,\ldots,b_n+1\}$. Finite empirical minima and maxima are continuous in $Q$, as are these endpoints and their midpoints. Composing with the continuous $h_{0r}$ proves continuity of the ordered vector $\beta_{n,r}(Q)$. Its norm bounds follow from $\norm{h_{0r}}_\infty\le C_h$. This argument concerns the coefficient centers, not continuity of hard routing.
\end{proof}

\begin{Bremark}
The side-length conclusion in Proposition~\ref{B:prop:dynamic_uniform_sparse}(ii) is the regular-split-net property used in earlier BART approximation arguments, including \cite{JeongRockova2023}.  Here it is not imposed through a separate abstract definition; it follows from the explicit uniform split-net construction.
\end{Bremark}

\subsubsection{Auxiliary inputs}\label{B:subsection:2.6}

We collect the probabilistic inputs used later.  The inputs are a Gaussian Euclidean-ball bound, a tree-size tail, an active-row Haar bound on $\SO(p)$, and a binary sign-pattern bound.  Exponentially powerful local tests for empirical $L_2$ loss and unknown Gaussian variance are supplied by Theorem~3 of \cite{Jeong2025}, which we quote in the form needed here.
\begin{Blemma}[Gaussian Euclidean-ball lower bound]
\label{B:lem:gaussian_small_ball}
Let $B\sim N(0,\Sigma)$ in $\R^m$, where every eigenvalue of $\Sigma$ lies in $[\underline\lambda,\overline\lambda]$ with $0<\underline\lambda\le\overline\lambda<\infty$.  Then there exists a constant $C>0$, depending only on $(\underline\lambda,\overline\lambda)$, such that for every $b\in\R^m$ and every $r\in(0,1)$,
\[
-\log \Pr(\norm{B-b}_2\le r)
\le C\{m\log(m/r)+\norm{b}_2^2+1\}.
\]
\end{Blemma}

\begin{proof}
Write $\Sigma^{1/2}$ for the symmetric square root of $\Sigma$ and set $Z=\Sigma^{-1/2}B\sim N(0,I_m)$ and $\mu=\Sigma^{-1/2}b$.  Take $c_1=\underline\lambda^{-1/2}$ and $c_2=\max\{1,\sqrt{\overline\lambda}\}$, so that $\norm{\Sigma^{1/2}u}_2\le c_2\norm u_2$ and $\norm{\Sigma^{-1/2}u}_2\le c_1\norm u_2$.  Hence
\[
\Pr(\norm{B-b}_2\le r)
\ge \Pr(\norm{Z-\mu}_2\le r/c_2).
\]
Let $\rho=r/c_2\in(0,1)$.  On $\{z:\norm z_2\le\rho\}$,
\[
\varphi_m(z+\mu)
\ge (2\pi)^{-m/2}\exp\{-(\norm\mu_2+\rho)^2/2\}.
\]
The ball volume is $\pi^{m/2}\rho^m/\Gamma(m/2+1)$. Consequently, $-\log\Pr(\norm{B-b}_2\le r)\le (m/2)\log2+\log\Gamma(m/2+1)+m\log(c_2/r)+(\norm\mu_2+\rho)^2/2$. Applying Stirling's bound to the same density--volume calculation yields
\[
-\log \Pr(\norm{Z-\mu}_2\le\rho)
\le C_1\{m\log(m/\rho)+\norm\mu_2^2+1\}.
\]
The claim follows from $\norm\mu_2^2\lesssim\norm b_2^2$ and $\rho\asymp r$.
\end{proof}

\begin{Bproposition}[Gaussian contraction criterion with unknown variance, quoted]
\label{B:prop:jeong_contraction}
For a regression function $f$, let
\[
\mu_f=(f(x_{1,n}),\ldots,f(x_{n,n}))^\top\in\R^n,
\]
and define
\begin{equation}
d_G^2((\mu,\sigma),(\mu',\sigma'))
=
\frac1n\norm{\mu-\mu'}_2^2+|\sigma-\sigma'|^2.
\label{B:eq:jeong_metric}
\end{equation}
Let $\widetilde\varepsilon_n>0$ satisfy $\widetilde\varepsilon_n/\sigma_0\to0$ and $n\widetilde\varepsilon_n^2/\sigma_0^2\to\infty$.  Let $\overline\Pi_n$ be a prior on $(\mu,\sigma)\in\R^n\times(0,\infty)$ and let $\overline\Theta_n$ be a measurable sieve.  Suppose that, for finite constants $C,D>0$ and a sufficiently large constant $E>0$,
\[
\overline\Pi_n\left\{
\norm{\mu-\mu_{f_0}}_2^2\le n\widetilde\varepsilon_n^2,
\quad
|\sigma^2-\sigma_0^2|\le\sigma_0\widetilde\varepsilon_n
\right\}
\ge
\exp\left\{-C\frac{n\widetilde\varepsilon_n^2}{\sigma_0^2}\right\},
\]
\[
\log N(\widetilde\varepsilon_n,\overline\Theta_n,d_G)
\le
D\frac{n\widetilde\varepsilon_n^2}{\sigma_0^2},
\qquad
\overline\Pi_n(\overline\Theta_n^c)
\le
\exp\left\{-E\frac{n\widetilde\varepsilon_n^2}{\sigma_0^2}\right\}.
\]
Then there exists $M>0$ such that
\[
\E_0\overline\Pi_n\left
(d_G((\mu,\sigma),(\mu_{f_0},\sigma_0))>M\widetilde\varepsilon_n\mid Y^{(n)}
\right)\to0.
\]
This is Theorem~3 of \cite{Jeong2025}, specialized to the fixed-design mean vector $\mu_f$. In its proof, the required tail exponent $E$ depends on the local prior-mass constant and local Gaussian bounds, not on the entropy constant $D$. Once the sieve is fixed, $D$ enters the choice of the final multiplier $M$.
\end{Bproposition}

\begin{Blemma}[Exact Gaussian Kullback--Leibler formulas]
\label{B:lem:gaussian_kl}
For model \eqref{B:eq:model},
\[
\KL(P_0^{(n)},P_{f,\sigma}^{(n)})
=
\frac n2\left\{
\log\frac{\sigma^2}{\sigma_0^2}
+
\frac{\sigma_0^2}{\sigma^2}-1
+
\frac{\norm{f-f_0}_n^2}{\sigma^2}
\right\},
\]
and the centered second moment of the log-likelihood ratio is
\[
\Var(P_0^{(n)},P_{f,\sigma}^{(n)})
=
\frac n2\left(\frac{\sigma_0^2}{\sigma^2}-1\right)^2
+
\frac{n\sigma_0^2}{\sigma^4}\norm{f-f_0}_n^2.
\]
Consequently, in a fixed neighborhood of $\sigma_0>0$,
\[
\KL(P_0^{(n)},P_{f,\sigma}^{(n)})
+
\Var(P_0^{(n)},P_{f,\sigma}^{(n)})
\lesssim
n\{\norm{f-f_0}_n^2+|\sigma-\sigma_0|^2\}.
\]
\end{Blemma}

\begin{proof}
The formulas follow by expanding the Gaussian log-likelihood ratio observation by observation.  They are also the direct calculations used in the proof of Theorem~3 of \cite{Jeong2025}.  The final inequality follows by Taylor expansion around $\sigma=\sigma_0$.
\end{proof}

\begin{Blemma}[Tree-size tail under the $\nu$-depth prior]
\label{B:lem:tree_tail}
Under \textnormal{(P3)}, there are constants $c_K>0$ and $k_0\ge2$, depending only on $\nu$, such that for every $n$, tree $t$, rotation $Q_t\in\SO(p)$, and integer $k\ge k_0$,
\[
\mathbb P_n(K_t>k\mid Q_t)\le e^{-c_Kk\log k}.
\]
The same bound holds after integrating over $Q_t$.
\end{Blemma}

\begin{proof}
Fix $n$ and $Q_t$. Every realized tree is a finite full binary tree: along any path, a coordinate cannot reuse a cutpoint, so the depth is at most $p b_n$. Fix a full binary tree shape $\tau$ with $K$ leaves.  Let $I(\tau)$ be the internal nodes, let $L(\tau)$ be the leaves, and let $d(u)$ be the depth of a node $u$.  Pre-generate independent split-attempt coins at all potential nodes. Realizing this shape requires a successful coin at each of its prescribed internal nodes; unavailable labels can only terminate a path earlier. Dropping all other requirements bounds the probability of this unlabeled shape by
\[
\prod_{u\in I(\tau)}\nu^{d(u)+1}
=\nu^{\sum_{u\in I(\tau)}(d(u)+1)}.
\]
Let $I_\tau=\sum_{u\in I(\tau)}d(u)$ and $E_\tau=\sum_{v\in L(\tau)}d(v)$.  For full binary trees, $E_\tau=I_\tau+2(K-1)$.  If the leaf depths are $d_1,\dots,d_K$, Kraft's equality gives $\sum_{j=1}^K2^{-d_j}=1$, and Jensen's inequality implies $E_\tau\ge K\log_2K$.  Therefore
\[
\sum_{u\in I(\tau)}(d(u)+1)=I_\tau+K-1=E_\tau-K+1\ge K\log_2K-K+1.
\]
The number of full binary tree shapes with $K$ leaves is the Catalan number $C_{K-1}\le4^{K-1}$.  Hence
\[
\mathbb P_n(K_t=K\mid Q_t)
\le 4^{K-1}\nu^{K\log_2K-K+1}
\le e^{-cK\log K}
\]
for all sufficiently large $K$, with constants depending only on $\nu$. Summing over integers $K>k$ and decreasing the exponential constant gives the conditional tail bound uniformly in $n$ and $Q_t$. Integrating over $Q_t$ gives its marginal version.
\end{proof}

\begin{Blemma}[Active-row Haar small ball]
\label{B:lem:active_haar}
Let \(Q\sim \mathrm{Haar}(\SO(p))\), let \(S\subset\{1,\ldots,p\}\) with
\(|S|=d\), and assume \(p\ge 2\), \(1\le d\le p\).  For every fixed
\(Q_0\in\SO(p)\) and all sufficiently small \(r\in(0,1)\),
\[
-\log \Pr\left(
  \|Q_{S,\cdot}-Q_{0,S,\cdot}\|_F\le r
\right)
\le
C\left\{
D_{p,d}\log(C/r)+D_{p,d}\log(eD_{p,d})
\right\}
\]
where
\[
D_{p,d}=dp-\frac{d(d+1)}2.
\]
In particular, when $p$ and $d$ are fixed and $\log(1/r)\lesssim\log n$, the negative logarithm is $O(\log n)$.
\end{Blemma}

\begin{proof}
Let $H$ and $H_+$ denote normalized Haar measure on $O(p)$ and $\SO(p)$, respectively, and put $U_0=Q_{0,S,\cdot}$.
Let $G\in\R^{d\times p}$ have independent $N(0,1)$ entries and define
\[
U(G)=(GG^\top)^{-1/2}G\in\St(d,p),
\qquad
\St(d,p)=\{U\in\R^{d\times p}:UU^\top=I_d\}.
\]
The matrix $G$ has full row rank almost surely. For every $R\in O(p)$, $GR\stackrel d=G$ and $U(GR)=U(G)R$. The right action of $O(p)$ on $\St(d,p)$ is transitive, so $U(G)$ has the invariant Stiefel distribution, which is also the law of $Q_{S,\cdot}$ under $H$.

Write $G=U_0+E$ and $A=GG^\top$. Since $\norm{U_0}_{\mathrm{op}}=1$,
\[
A-I_d=U_0E^\top+EU_0^\top+EE^\top,
\qquad
\norm{A-I_d}_F\le2\norm E_F+\norm E_F^2.
\]
For $\norm E_F$ below a fixed numerical constant, the eigenvalues of $A$ lie in $[1/2,3/2]$. The identity
\[
A^{-1/2}-I_d=A^{-1/2}(I_d-A)(I_d+A^{1/2})^{-1}
\]
gives $\norm{A^{-1/2}-I_d}_F\le C\norm E_F$. Hence
\begin{align*}
\norm{U(G)-U_0}_F
&=\norm{(A^{-1/2}-I_d)U_0+A^{-1/2}E}_F\\
&\le\norm{A^{-1/2}-I_d}_F+\sqrt2\norm E_F
\le C\norm E_F.
\end{align*}
Thus, for a fixed numerical $c\in(0,1)$ and all sufficiently small $r$,
\[
H\{Q:\norm{Q_{S,\cdot}-Q_{0,S,\cdot}}_F\le r\}
\ge\Pr(\norm{G-U_0}_F\le cr).
\]
Vectorize $G$ in $\R^{dp}$. Its covariance is the identity, $\norm{U_0}_F^2=d$, and $cr<1$. Lemma~\ref{B:lem:gaussian_small_ball} therefore yields
\[
-\log\Pr(\norm{G-U_0}_F\le cr)
\le C\{dp\log(dp/(cr))+d+1\}.
\]
Since
\[
D_{p,d}\ge\frac{d(p-1)}2,
\qquad
dp\le\frac{2p}{p-1}D_{p,d}\le4D_{p,d},
\]
and $d+1\le C D_{p,d}$ for $p\ge2$, this proves the asserted bound under $H$.

Left multiplication by a reflection preserves $H$ and exchanges its two determinant components, each of mass $1/2$; its normalized restriction to $\SO(p)$ is $H_+$. If $d<p$, choose $j\notin S$ and put $D=I_p-2e_je_j^\top$. Multiplication by $D$ exchanges those components and leaves $Q_{S,\cdot}$ unchanged. Consequently,
\[
H_+\{Q:\norm{Q_{S,\cdot}-Q_{0,S,\cdot}}_F\le r\}
=H\{Q:\norm{Q_{S,\cdot}-Q_{0,S,\cdot}}_F\le r\}.
\]
If $d=p$ and $\det Q=-1$, then $QQ_0^\top$ is orthogonal with determinant $-1$. Its nonreal eigenvalues occur in conjugate pairs, so it has an eigenvalue $-1$ and
\[
\norm{Q-Q_0}_F=\norm{QQ_0^\top-I_p}_F
\ge\norm{QQ_0^\top-I_p}_{\mathrm{op}}\ge2.
\]
For $r<1$, the full-matrix ball lies in $\SO(p)$, and its $H_+$-probability is twice its $H$-probability. Thus the bound transfers to $H_+$ in both cases. For fixed $p,d$, it is $O\{\log(1/r)\}$, giving the final assertion.
\end{proof}

\begin{Blemma}[Linear sign-pattern bound]
\label{B:lem:warren}
Let $m,k\ge1$ and let $a_1,\ldots,a_m$ be affine functions on
$\mathbb R^k$.
For any $A\subset\mathbb R^k$, the number of distinct binary vectors
$(\mathbf1\{a_i(z)\ge0\})_{i=1}^m$, $z\in A$,
is at most $\sum_{j=0}^{\min(m,k)}\binom mj$.
In particular, it is at most $(em/k)^k$ when $m\ge k$,
and at most $2^m$ when $m<k$.
Consequently, if $m\le n^Ck$ for a fixed $C>0$, then the logarithm
of this number is $O(k\log n)$ for $n\ge2$.
\end{Blemma}

\begin{proof}
The claim is immediate if $A$ is empty.
Otherwise, choose one point of $A$ for each realized binary vector.
There are finitely many such points. Choose $\epsilon>0$ smaller
than the absolute value of every nonzero $a_i$ at these points,
with arbitrary $\epsilon>0$ if all these values are zero.
At every chosen point, all $a_i+\epsilon$ are nonzero, and
$\mathbf1\{a_i\ge0\}=\mathbf1\{a_i+\epsilon>0\}$.
Thus the original binary vectors are realized as strict sign
patterns of one fixed perturbed affine arrangement.

Let $N(m,k)$ be the maximal number of regions determined by at most
$m$ affine hyperplanes in $\mathbb R^k$.
Adding one hyperplane creates at most $N(m-1,k-1)$ new regions,
so $N(m,k)\le N(m-1,k)+N(m-1,k-1)$.
With $N(0,k)=N(m,0)=1$, induction gives
$N(m,k)\le\sum_{j=0}^{\min(m,k)}\binom mj$.
Each strict sign pattern is a region, which proves the first bound.
For $m\ge k$,
$\sum_{j=0}^k\binom mj\le(m/k)^k(1+k/m)^m\le(em/k)^k$;
for $m<k$, the sum equals $2^m$.
The logarithmic bound follows by considering these two cases.
\end{proof}

\Needspace{10\baselineskip}
\subsection{Main theorem}\label{B:section:3}

The following is the posterior contraction theorem of the main paper, restated with the prior and design conditions above.

\begin{Btheorem}[Fixed-dimensional additive posterior contraction for RoBART]
\label{B:thm:main}
Assume \eqref{B:eq:bounded_design}, \eqref{B:eq:truth}, the componentwise H\"older and boundedness conditions in Definition~\ref{B:def:holder}, the prior conditions \textnormal{(P1)}--\textnormal{(P5)}, and Assumption~\ref{B:ass:dynamic_uniform}.  Let $\epsn$ be given by \eqref{B:eq:epsn}.  Then there exists $M>0$ such that
\begin{equation}
\E_0\Pi_n\left(\norm{f-f_0}_n+|\sigma-\sigma_0|>M\epsn\mid Y^{(n)}\right)\to0.
\label{B:eq:main_contraction}
\end{equation}
Consequently, because $\sigma_0$ is fixed and strictly positive,
\[
\E_0\Pi_n\left(\norm{f-f_0}_n+|\sigma^2-\sigma_0^2|>M'\epsn\mid Y^{(n)}\right)\to0
\]
for a sufficiently large constant $M'>0$.
\end{Btheorem}

\begin{Bremark}[Interpretation of the rate]
For fixed $R$ and component parameters, \eqref{B:eq:epsn} gives
\begin{align*}
\epsn\asymp\max_{1\le r\le R}
\left(\frac{\log n}{n}\right)^{\bar\alpha_r/(2\bar\alpha_r+d_r)}.
\end{align*}
The factors involving $\lambda_rd_r$ are constants in $n$, while $d_r$ remains in the exponents. Localizing the rotations costs $O(\log n)$. The prescribed uniform split labels contribute $\sum_r(K_{n,r}-1)\log p=O(\sum_rK_{n,r})=o(n\epsn^2)$ to the negative log bound in Proposition~\ref{B:prop:goodset_mass}. These costs are lower order than the tree approximation cost in \eqref{B:eq:balance}. This is a function-level upper bound; it does not assert recovery of the rotations, active sets, or the possibly nonidentified additive decomposition.
\end{Bremark}

\subsection{Approximation and local rotation control}\label{B:section:4}

\subsubsection{Active-row rotation perturbation}\label{B:subsection:4.1}

The first deterministic ingredient controls how much a component template changes when the active rows of its rotation matrix are perturbed.

\begin{Blemma}[Componentwise sparse active-row rotation perturbation]
\label{B:lem:rotation_perturbation}
For every $r\le R$ and every $Q,Q'\in \SO(p)$,
\[
\norm{h_{0r}(\pi_{S_r}(Q\cdot))-h_{0r}(\pi_{S_r}(Q'\cdot))}_n
\le
C_r\lambda_rd_r
\norm{Q_{S_r,\cdot}-Q'_{S_r,\cdot}}_F^{\alpha_{r,\min}},
\]
where $C_r$ depends only on $M_X$ and the fixed smoothness vector.  In particular, if $Q\in\mathcal Q_{n,r}$, then
\[
\norm{h_{0r}(\pi_{S_r}(Q\cdot))-h_{0r}(\pi_{S_r}(Q_{0r}\cdot))}_n
\lesssim\varepsilon_{n,r}.
\]
\end{Blemma}

\begin{proof}
For each design point $x_{i,n}$, H\"older continuity gives
\[
|h_{0r}(\pi_{S_r}(Qx_{i,n}))-h_{0r}(\pi_{S_r}(Q'x_{i,n}))|
\le
\lambda_r\sum_{k=1}^{d_r}|e_{s_{rk}}^\top(Q-Q')x_{i,n}|^{\alpha_{rk}}.
\]
By Cauchy's inequality and \eqref{B:eq:bounded_design},
\[
|e_{s_{rk}}^\top(Q-Q')x_{i,n}|
\le\norm{(Q-Q')_{s_{rk},\cdot}}_2\norm{x_{i,n}}_2
\le M_X\norm{Q_{S_r,\cdot}-Q'_{S_r,\cdot}}_F.
\]
Taking empirical norms bounds the difference in the lemma by $\lambda_r\sum_{k=1}^{d_r}M_X^{\alpha_{rk}}\norm{(Q-Q')_{s_{rk},\cdot}}_2^{\alpha_{rk}}$. Every selected row is a unit vector, so $\norm{Q_{S_r,\cdot}-Q'_{S_r,\cdot}}_F\le2\sqrt{d_r}$. For positive distance, each power $\alpha_{rk}$ is bounded by the power $\alpha_{r,\min}$ times $(2\sqrt{d_r})^{\alpha_{rk}-\alpha_{r,\min}}$; at zero distance the claim is immediate. Hence the last bound is at most $C_r\lambda_rd_r\norm{Q_{S_r,\cdot}-Q'_{S_r,\cdot}}_F^{\alpha_{r,\min}}$, which proves the first claim.  The second follows from \eqref{B:eq:deltan} with $c_{\delta,r}$ sufficiently small.
\end{proof}

\subsubsection{Approximation in the additive rotated coordinates}\label{B:subsection:4.2}

The approximation of component $r$ is built on $\mathcal P_{n,r}(Q)$ in the active empirical box $D_{Q,r}$. Coordinates outside $S_r$ are left unsplit. Applying the same split rules beyond this box extends $\widehat h_{r,Q}$ to $\R^{d_r}$ and gives the corresponding $p$-dimensional tree function after rotation. Only values at the design points enter the approximation bound.

\begin{Blemma}[Componentwise rotated anisotropic approximation]
\label{B:lem:approx_local_ball}
Under Assumption~\ref{B:ass:dynamic_uniform}, let $\widehat h_{r,Q}$ be the step function in Proposition~\ref{B:prop:dynamic_uniform_sparse}. For all sufficiently large $n$, every $r\le R$, and every $Q\in\mathcal Q_{n,r}$,
\[
\norm{h_{0r}(\pi_{S_r}(Q_{0r}\cdot))-\widehat h_{r,Q}(\pi_{S_r}(Q\cdot))}_n
\lesssim\varepsilon_{n,r}.
\]
\end{Blemma}

\begin{proof}
By Proposition~\ref{B:prop:dynamic_uniform_sparse}(iii), the uniform approximation error on $D_{Q,r}$ is $O(\varepsilon_{n,r})$. For every design point, $\pi_{S_r}(Qx_{i,n})\in D_{Q,r}$ by definition, and hence
\[
\norm{h_{0r}(\pi_{S_r}(Q\cdot))-\widehat h_{r,Q}(\pi_{S_r}(Q\cdot))}_n
\le
\norm{h_{0r}-\widehat h_{r,Q}}_{\infty,D_{Q,r}}
\lesssim\varepsilon_{n,r}.
\]
Lemma~\ref{B:lem:rotation_perturbation} controls the difference between $Q$ and $Q_{0r}$.  The triangle inequality proves the claim.
\end{proof}

For every tuple $(Q_1,\ldots,Q_R)$ with $Q_r\in\mathcal Q_{n,r}$, summing Lemma~\ref{B:lem:approx_local_ball} over the fixed number of components gives
\[
\norm{f_0-\sum_{r=1}^R\widehat h_{r,Q_r}(\pi_{S_r}(Q_r\cdot))}_n
\lesssim
\sum_{r=1}^R\varepsilon_{n,r}
=\epsn.
\]

\subsection{Kullback--Leibler prior mass}\label{B:section:5}

The prior-mass lower bound is built from a good set in which the first $R$ trees carry the $R$ additive components and the remaining $T-R$ trees are nearly null.  This construction exploits the redundancy of the forest without requiring every tree-specific rotation to be close to a truth frame.

\subsubsection{\texorpdfstring{An $R$-signal-tree good set}{An R-signal-tree good set}}\label{B:subsection:5.1}

For $r\le R$, use the fixed state $\mathfrak t_{n,r}$, partition $\mathcal P_{n,r}(Q)$, and ordered coefficient center $\beta_{n,r}(Q)$ from Proposition~\ref{B:prop:dynamic_uniform_sparse}. In particular,
\begin{equation}
\norm{\beta_{n,r}(Q)}_\infty\le C_h,
\qquad
\norm{\beta_{n,r}(Q)}_2^2\le C_h^2K_{n,r}.
\label{B:eq:beta_bound}
\end{equation}
Use the same left-to-right leaf order for $M_r$ and $\beta_{n,r}(Q_r)$. Define the nested events
\[
A_Q=\bigcap_{r=1}^R\{Q_r\in\mathcal Q_{n,r}\},
\qquad
A_T=A_Q\cap\bigcap_{r=1}^R\{(\cT_r,V_r,C_r)=\mathfrak t_{n,r}\},
\]
\[
A_M=A_T\cap\bigcap_{r=1}^R
\{\norm{M_r-\beta_{n,r}(Q_r)}_2\le c_1\varepsilon_{n,r}\}.
\]
The coefficient comparisons are evaluated only on $A_T$, so their centers are defined and their dimensions agree. For $t>R$, let $\mu_t$ denote the unique leaf mean when $K_t=1$, and set
\[
A_{\mathrm{null}}=\bigcap_{t=R+1}^T\{K_t=1,\ |\mu_t|\le c_2\epsn/T\},
\qquad
A_\sigma=\{|\sigma^2-\sigma_0^2|\le c_3\sigma_0\epsn\}.
\]
Finally, let $A_n=A_M\cap A_{\mathrm{null}}\cap A_\sigma$. All constants $c_1,c_2,c_3$ are fixed and positive. Proposition~\ref{B:prop:dynamic_uniform_sparse}(iv) makes these events measurable. Empty intersections equal the whole generating space when $T=R$.

\begin{Bproposition}[Approximation on the good set]
\label{B:prop:goodset_approx}
For sufficiently small fixed constants $c_1,c_2,c_3>0$ and all sufficiently large $n$, every ensemble in $A_n$ induces a pair $(f,\sigma)$ satisfying
\[
\norm{f-f_0}_n\lesssim\epsn,
\qquad
|\sigma-\sigma_0|\lesssim\epsn.
\]
Consequently,
\[
\KL(P_0^{(n)},P_{f,\sigma}^{(n)})
+
\Var(P_0^{(n)},P_{f,\sigma}^{(n)})
\lesssim n\epsn^2.
\]
\end{Bproposition}

\begin{proof}
Write
\[
f(x)=\sum_{r=1}^Rg_r(Q_rx)+\sum_{t=R+1}^T\mu_t,
\]
where $g_r$ uses the discrete state $\mathfrak t_{n,r}$ under $Q_r$ and leaf vector $M_r$.  Then
\begin{align*}
\norm{f-f_0}_n
&\le
\sum_{r=1}^R
\norm{g_r(Q_r\cdot)-\widehat h_{r,Q_r}(\pi_{S_r}(Q_r\cdot))}_n\\
&\quad+
\sum_{r=1}^R
\norm{\widehat h_{r,Q_r}(\pi_{S_r}(Q_r\cdot))-h_{0r}(\pi_{S_r}(Q_{0r}\cdot))}_n
+
\sum_{t=R+1}^T|\mu_t|.
\end{align*}
For the first sum, let $d_{M,r}=M_r-\beta_{n,r}(Q_r)$ and let $n_{r,k}$ be the number of design points falling in leaf $k$.  Since the two functions live on the same partition,
\[
\norm{g_r(Q_r\cdot)-\widehat h_{r,Q_r}(\pi_{S_r}(Q_r\cdot))}_n^2
=
\frac1n\sum_{k=1}^{K_{n,r}}n_{r,k}d_{M,r,k}^2
\le\norm{M_r-\beta_{n,r}(Q_r)}_2^2.
\]
The second sum is $O(\sum_r\varepsilon_{n,r})$ by Lemma~\ref{B:lem:approx_local_ball}.  The null-tree sum is $O(\epsn)$ because $T$ is fixed.  This proves the function bound. On $A_\sigma$, $|\sigma-\sigma_0|=|\sigma^2-\sigma_0^2|/(\sigma+\sigma_0)\le c_3\epsn$. Lemma~\ref{B:lem:gaussian_kl} then gives the final assertion.
\end{proof}

\begin{Bproposition}[Prior mass of the good set]
\label{B:prop:goodset_mass}
Under Assumption~\ref{B:ass:dynamic_uniform}, there exists $C>0$ such that
\[
\mathbb P_n(A_n)\ge e^{-Cn\epsn^2}
\]
for all sufficiently large $n$.
\end{Bproposition}

\begin{proof}
Fix $r\le R$ and $Q_r\in\mathcal Q_{n,r}$. The state $\mathfrak t_{n,r}$ has $K_{n,r}=2^{L_{n,r}}$ leaves and $K_{n,r}-1$ internal nodes. Every prescribed label and cut is available by Proposition~\ref{B:prop:dynamic_uniform_sparse}(ii). In the discrete mass formula of Section~\ref{B:subsection:2.3}, $a_\eta\le b_n$ and each terminal factor is at least $1-\nu$. The complete-depth topology satisfies $\sum_{\ell=0}^{L_{n,r}-1}(\ell+1)2^\ell=(L_{n,r}-1)K_{n,r}+1$. Hence $\mathbb P_n\{(\cT_r,V_r,C_r)=\mathfrak t_{n,r}\mid Q_r\}\ge(1-\nu)^{K_{n,r}}\nu^{(L_{n,r}-1)K_{n,r}+1}(pb_n)^{-(K_{n,r}-1)}$. Taking negative logarithms gives $-\log\mathbb P_n\{(\cT_r,V_r,C_r)=\mathfrak t_{n,r}\mid Q_r\}\le C K_{n,r}(L_{n,r}+\log b_n+1)+(K_{n,r}-1)\log p\le C K_{n,r}\log n$, using fixed $p$ and \textnormal{(D2)}.

Given the actual rotations and discrete states, \textnormal{(P4)} gives $M_r\sim N(0,\tau^2I_{K_{n,r}})$ on the prescribed state. By \eqref{B:eq:beta_bound} and Lemma~\ref{B:lem:gaussian_small_ball}, $\mathbb P_n\{\norm{M_r-\beta_{n,r}(Q_r)}_2\le c_1\varepsilon_{n,r}\mid Q_r,(\cT_r,V_r,C_r)=\mathfrak t_{n,r}\}\ge\exp(-C K_{n,r}\log n)$ uniformly over $Q_r\in\mathcal Q_{n,r}$. Here $\log K_{n,r}+\log(1/\varepsilon_{n,r})=O(\log n)$, and the fixed conditional Gaussian kernel depends on the structural state only through its leaf count.

Let $H$ be normalized Haar measure on $\SO(p)$. To integrate the moving centers, tree independence gives $\mathbb P_n(A_M)=\int_{\prod_{r=1}^R\mathcal Q_{n,r}}\prod_{r=1}^R\mathbb P_n\{(\cT_r,V_r,C_r)=\mathfrak t_{n,r},\ \norm{M_r-\beta_{n,r}(Q_r)}_2\le c_1\varepsilon_{n,r}\mid Q_r\}\,\prod_{r=1}^R H(dQ_r)$, where each coefficient comparison is evaluated only on the prescribed state. The two uniform bounds above imply $\mathbb P_n(A_M)\ge\exp\{-C\sum_{r=1}^R K_{n,r}\log n\}\mathbb P_n(A_Q)$. By \textnormal{(P1)} and Lemma~\ref{B:lem:active_haar}, $-\log\mathbb P_n(A_Q)\lesssim\sum_r\log(1/\delta_{n,r})\lesssim\log n$.

For $t>R$, not attempting a root split gives $\mathbb P_n(K_t=1\mid Q_t)\ge1-\nu$. For $Z\sim N(0,\tau^2)$ and $0<a\le1$, $\Pr(|Z|\le a)\ge2a(2\pi\tau^2)^{-1/2}\exp\{-1/(2\tau^2)\}$. With $a=c_2\epsn/T$, this yields $\mathbb P_n\{K_t=1,\ |\mu_t|\le c_2\epsn/T\mid Q_t\}\ge c\epsn$ uniformly in $Q_t$. Integrating over these unrestricted rotations and using independence across trees gives $\mathbb P_n(A_{\mathrm{null}})\ge(c\epsn)^{T-R}$. The bound is one when $T=R$.

Let $p_V$ be the prior density of $V=\sigma^2$. For sufficiently large $n$, the variance interval in $A_\sigma$ lies in $[\sigma_0^2/2,3\sigma_0^2/2]$. Positivity and continuity of $p_V$ give $\mathbb P_n(A_\sigma)\ge2c_3\sigma_0\epsn\inf_{v\in[\sigma_0^2/2,3\sigma_0^2/2]}p_V(v)\ge c\epsn$.

The signal-tree event, the remaining-tree event, and the noise event concern independent prior blocks under \textnormal{(P1)}--\textnormal{(P5)}. Multiplying the established bounds, $-\log\mathbb P_n(A_n)\le C\sum_{r=1}^R K_{n,r}\log n+C\log n\lesssim n\epsn^2$, by \eqref{B:eq:balance} and \eqref{B:eq:rate_arithmetic}.
\end{proof}

\Needspace{10\baselineskip}
\begin{Bcorollary}[Kullback--Leibler and Gaussian local prior mass]
\label{B:cor:KL_mass}
There exists a constant $A_0\ge1$ such that
\[
\Pi_n\left\{
\norm{f-f_0}_n\le A_0\epsn,
\quad
|\sigma^2-\sigma_0^2|\le\sigma_0A_0\epsn
\right\}
\ge e^{-Cn\epsn^2}.
\]
The same set is contained in a Kullback--Leibler neighborhood satisfying
\[
\KL(P_0^{(n)},P_{f,\sigma}^{(n)})
\vee
\Var(P_0^{(n)},P_{f,\sigma}^{(n)})
\le Cn\epsn^2.
\]
\end{Bcorollary}

\begin{proof}
Proposition~\ref{B:prop:goodset_approx} gives $\norm{f-f_0}_n\le C\epsn$ on $A_n$, while $A_\sigma$ directly gives $|\sigma^2-\sigma_0^2|\le c_3\sigma_0\epsn$. Choose $A_0\ge\max\{1,C,c_3\}$. The good set lies in the preimage of the displayed neighborhood, so Proposition~\ref{B:prop:goodset_mass} gives the stated bound for the induced prior $\Pi_n$. For every pair in that entire neighborhood, $|\sigma-\sigma_0|=|\sigma^2-\sigma_0^2|/(\sigma+\sigma_0)\le A_0\epsn$. This places $\sigma$ in a fixed positive neighborhood of $\sigma_0$ for all sufficiently large $n$, and Lemma~\ref{B:lem:gaussian_kl} gives both remaining bounds.
\end{proof}

\subsection{Direct entropy of the additive rotated forest sieve}\label{B:section:6}

For fixed topology and split labels, empirical routing is determined by affine comparisons in the tree-specific rotation matrices and thresholds. We count these routing patterns and then cover the leaf coefficients.

\subsubsection{The sieve}\label{B:subsection:6.1}

Let $K_+=\sum_{t=1}^TK_t$ be the total number of leaves across the forest.  Define
\begin{align*}
\Theta_n=\Frot(\bar K_n,M_n)\times I_n,
\end{align*}
where
\[
\bar K_n=\left\lceil M_K\frac{n\epsn^2}{\log n}\right\rceil,
\qquad
M_n=n^{A_\mu},
\qquad
I_n=[n^{-A_\sigma},e^{A_\sigma n\epsn^2}],
\]
and
\[
\Frot(\bar K_n,M_n)
=
\left\{
 f(x)=\sum_{t=1}^Tg_t(Q_tx):Q_t\in \SO(p),\ K_+\le\bar K_n,\ \max_{t\le T}\norm{M_t}_\infty\le M_n
\right\}.
\]
The coefficient constraint bounds every leaf mean. Membership means that at least one representation satisfying these bounds and the declared grid and admissibility rules exists. No split-label support restriction is needed because $p$ is fixed.

For use with Proposition~\ref{B:prop:jeong_contraction}, define the image sieve
\[
\overline\Theta_n
=
\{(\mu_f,\sigma):(f,\sigma)\in\Theta_n\}.
\]
Let $\mathfrak R_n$ be the finite set of routing configurations realized by this sieve, including topology and leaf order. For $\mathcal R\in\mathfrak R_n$, stack the leaf vectors $M_t$ in tree-by-tree order as $\beta$, and let $H_{\mathcal R}$ be the leaf-incidence matrix. Its $(i,(t,q))$ entry is one exactly when observation $i$ reaches leaf $q$ of tree $t$. Thus $\mu_f=H_{\mathcal R}\beta$, and $H_{\mathcal R}$ has $K_+(\mathcal R)$ columns. The image sieve is exactly $\overline\Theta_n=\bigcup_{\mathcal R\in\mathfrak R_n}\{H_{\mathcal R}[-M_n,M_n]^{K_+(\mathcal R)}\}\times I_n$: every coefficient vector in a cube can be paired with a fixed valid realization of that routing. This finite union of compact linear images, crossed with $I_n$, is compact and Borel measurable. No continuity of hard routing in $Q_t$ is used.

\subsubsection{Routing maps generated by shared rotations}\label{B:subsection:6.2}

Fix a forest topology with $I_+$ internal nodes and $K_+$ leaves, where $I_+=K_+-T\le\bar K_n$.  Also fix a split-label sequence $j(\ell)\in\{1,\ldots,p\}$ for internal nodes $\ell=1,\ldots,I_+$.  For node $\ell$, let $t(\ell)$ denote the tree containing it.  Conditional on topology and labels, the routing of the design points is determined by the binary indicators
\begin{equation}
\1\{e_{j(\ell)}^\top Q_{t(\ell)}x_{i,n}-\tau_\ell\ge0\},
\qquad i=1,\dots,n,
\quad \ell=1,\dots,I_+,
\label{B:eq:routing_signs}
\end{equation}
where $\tau_\ell$ is the numerical threshold at node $\ell$. Recording comparisons even at nodes that an observation does not visit can only enlarge the collection of routing descriptions.

The quantities $e_{j(\ell)}^\top Q_{t(\ell)}$ are rows of shared tree-specific rotations, not independent node-specific normals.  For an upper bound, we ignore the orthogonality constraints and allow every $Q_t$ to range freely over $\R^{p\times p}$; we also allow all thresholds to be arbitrary real numbers.  Both enlargements can only increase the number of sign patterns.

\begin{Blemma}[Direct routing count for rotation forests]
\label{B:lem:direct_routing}
For $n\ge2$ and fixed topology and split-label sequence,
the number of empirical routing maps generated by tree-specific
rotations and thresholds is at most
$\exp\{C(p^2+I_+)\log n\}$,
where $C$ may depend on the fixed $T$ and $p$.
\end{Blemma}

\begin{proof}
If $I_+=0$, there is one routing map.
Suppose $I_+\ge1$.
The routing decisions are the binary indicators in \eqref{B:eq:routing_signs},
with equality assigned to the right child.
Allowing $Q_1,\ldots,Q_T$ to range over $\mathbb R^{p\times p}$
and all thresholds to range over $\mathbb R$ can only increase
the number of routing maps.
The $m=nI_+$ comparisons are affine functions of
$k=Tp^2+I_+$ real variables.
If $m\ge k$, Lemma~\ref{B:lem:warren} bounds the number of binary patterns by
$(em/k)^k\le(en)^k$, since $m/k\le n$.
If $m<k$, the bound $2^m\le2^k$ gives the same logarithmic order.
Each binary pattern determines all leaf assignments in the fixed
forest topology. Since $T$ is fixed, the stated bound follows.
\end{proof}

\begin{Bproposition}[Entropy of the additive rotation forest sieve]
\label{B:prop:sieve_entropy}
For every fixed $C>0$,
\[
\log N\left(C\epsn,\Frot(\bar K_n,M_n),\norm{\cdot}_n\right)
\lesssim
\bar K_n\log n+p^2\log n.
\]
Consequently,
\[
\log N\left(C\epsn,\Frot(\bar K_n,M_n),\norm{\cdot}_n\right)
\lesssim n\epsn^2.
\]
\end{Bproposition}

\begin{proof}
We count the components of a cover.

\medskip
\noindent\textbf{Tree topologies and split labels.}  The number of rooted binary forest topologies with total leaf count at most $\bar K_n$ is at most $\exp\{C\bar K_n\}$, since $T$ is fixed and Catalan numbers are bounded by exponential functions.  A forest with $I_+$ internal nodes has at most $p^{I_+}$ split-label sequences.  Since $p$ is fixed, the topology and label cost is $\exp\{C\bar K_n\}$.

\medskip
\noindent\textbf{Routing maps from shared rotations.}  For each fixed topology and label sequence, Lemma~\ref{B:lem:direct_routing} gives at most
\[
\exp\{C(p^2+\bar K_n)\log n\}
\]
empirical routing maps. Free numerical thresholds already include the original cutpoint-index and frame-dependent-grid choices, so no additional $b_n^{I_+}$ factor is needed.

\medskip
\noindent\textbf{Leaf coefficients.} For each realized $\mathcal R$, select one valid rotation and discrete-tree realization. Every row of $H_{\mathcal R}$ contains exactly $T$ ones, including when some columns represent empty leaves. Thus $\norm{f_\beta-f_{\beta'}}_n\le\norm{H_{\mathcal R}(\beta-\beta')}_\infty\le T\norm{\beta-\beta'}_\infty$. An $\ell_\infty$ cover of $[-M_n,M_n]^{K_+}$ with radius $C\epsn/T$ and centers in that cube therefore gives an empirical cover of radius $C\epsn$. Pairing its centers with the selected valid realization keeps every covering function in the original sieve. The logarithmic covering size is at most $K_+\log\{1+2TM_n/(C\epsn)\}$, which is bounded by
\[
C\bar K_n\log(M_n/\epsn)
\lesssim\bar K_n\log n.
\]
Combining the three bounds gives the first assertion.  By definition, $\bar K_n\log n\lesssim n\epsn^2$, and fixed $p$ implies $p^2\log n=o(n\epsn^2)$ by \eqref{B:eq:rate_arithmetic}.
\end{proof}

\begin{Bcorollary}[Joint entropy including the noise scale]
\label{B:cor:joint_entropy}
For the metric $d_G$ in \eqref{B:eq:jeong_metric},
\[
\log N(C\epsn,\overline\Theta_n,d_G)\lesssim n\epsn^2.
\]
\end{Bcorollary}

\begin{proof}
Use radius $C\epsn/\sqrt2$ for both the function cover from Proposition~\ref{B:prop:sieve_entropy} and the scalar noise-scale cover. The interval $I_n=[n^{-A_\sigma},e^{A_\sigma n\epsn^2}]$ contributes
\[
\log N(C\epsn/\sqrt2,I_n,|\cdot|)
\lesssim
\log\left(\frac{e^{A_\sigma n\epsn^2}}{\epsn}\right)
\lesssim n\epsn^2.
\]
The empirical function metric is exactly the mean-vector part of $d_G$. The two squared covering errors add to at most $C^2\epsn^2$, and the logarithms of the cover sizes add.
\end{proof}

\subsection{Sieve complement prior mass}\label{B:section:7}

Because the rotation group is compact and $p$ is fixed, the complement of the sieve involves only tree sizes, leaf magnitudes, and the noise scale.  No support-size restriction is needed.

\begin{Bproposition}[Sieve complement prior mass]
\label{B:prop:sieve_tail}
For every prescribed $C_{\mathrm{tail}}>0$, sufficiently large constants $M_K,A_\mu,A_\sigma>0$ satisfy
\[
\Pi_n(\Theta_n^c)\le e^{-C_{\mathrm{tail}}n\epsn^2},
\qquad
\overline\Pi_n(\overline\Theta_n^c)\le e^{-C_{\mathrm{tail}}n\epsn^2}
\]
for all sufficiently large $n$.
\end{Bproposition}

\begin{proof}
If none of the following generating-space events occurs, the induced pair lies in $\Theta_n$ and its mean-vector image lies in $\overline\Theta_n$. Both $\Pi_n(\Theta_n^c)$ and $\overline\Pi_n(\overline\Theta_n^c)$ are therefore at most $\mathbb P_n(E_{K,n})+\mathbb P_n(E_{\mu,n})+\mathbb P_n(E_{\sigma,n})$, where
\[
E_{K,n}=\{K_+>\bar K_n\},
\qquad
E_{\mu,n}=\{\max_{t\le T}\norm{M_t}_\infty>M_n,\ K_+\le\bar K_n\},
\qquad
E_{\sigma,n}=\{\sigma\notin I_n\}.
\]

\medskip
\noindent\textbf{Large trees.}  Since $T$ is fixed, the union bound and Lemma~\ref{B:lem:tree_tail} give
\[
\mathbb P_n(E_{K,n})
\le
C_T\exp\{-c\bar K_n\log\bar K_n\}.
\]
By \eqref{B:eq:tree_growth}, $\log\bar K_n\gtrsim\log n$, while $\bar K_n\asymp M_Kn\epsn^2/\log n$.  Taking $M_K$ sufficiently large gives
\[
\mathbb P_n(E_{K,n})\le e^{-(C_{\mathrm{tail}}+2)n\epsn^2}.
\]

\medskip
\noindent\textbf{Large leaf magnitudes.} Condition on $Q_1,\ldots,Q_T,(\cT_t,V_t,C_t)_{t=1}^T$ with $K_+\le\bar K_n$. By \textnormal{(P4)}, each conditional Gaussian variance equals $\tau^2$. The union bound gives a conditional probability at most $2\bar K_n\exp\{-M_n^2/(2\tau^2)\}$ that any leaf mean exceeds $M_n$ in absolute value. Integrating this uniform bound over ensembles with $K_+\le\bar K_n$ yields $\mathbb P_n(E_{\mu,n})\le2\bar K_n\exp\{-M_n^2/(2\tau^2)\}$. Since $n\epsn^2=o(n)$ and $\log\bar K_n=O(\log n)$, any fixed $A_\mu\ge1/2$ makes this bound at most $e^{-(C_{\mathrm{tail}}+2)n\epsn^2}$ for all sufficiently large $n$.

\medskip
\noindent\textbf{Variance tail.}  Write $V=\sigma^2\sim\operatorname{IG}(a_\sigma,b_\sigma)$.  The upper tail is polynomial in $V$, and therefore the exponential upper endpoint yields
\[
\mathbb P_n(\sigma>e^{A_\sigma n\epsn^2})
\le
C\exp\{-2a_\sigma A_\sigma n\epsn^2\}.
\]
For the lower tail, $1/V$ is gamma and
\[
\mathbb P_n(\sigma<n^{-A_\sigma})
=
\mathbb P_n(1/V>n^{2A_\sigma})
\le
\exp\{-c n^{2A_\sigma}\}.
\]
Taking $A_\sigma$ sufficiently large makes both terms at most $e^{-(C_{\mathrm{tail}}+2)n\epsn^2}$.

The union bound and the preceding image-space inequalities prove the proposition. These inequalities need not be equalities, since a function or mean vector may have more than one forest representation.
\end{proof}

\subsection{Proof of the main theorem}\label{B:section:8}

\begin{proof}[Proof of Theorem~\ref{B:thm:main}]
The likelihood depends on $(f,\sigma)$ only through $(\mu_f,\sigma)$. Its normalizing integral is positive and finite: the likelihood is bounded by $(2\pi)^{-n/2}\sigma^{-n}$, and the inverse-gamma prior gives $\E_{\mathbb P_n}(\sigma^{-n})=b_\sigma^{-n/2}\Gamma(a_\sigma+n/2)/\Gamma(a_\sigma)<\infty$. Bayes' formula therefore gives $\overline\Pi_n(B\mid Y^{(n)})=\Pi_n\{(f,\sigma):(\mu_f,\sigma)\in B\mid Y^{(n)}\}$ for every Borel set $B\subset\R^n\times(0,\infty)$.

\medskip
\noindent\textbf{Local prior mass.} Fix $A_0\ge1$ as in Corollary~\ref{B:cor:KL_mass}, and set $\widetilde\varepsilon_n=A_0\epsn$. Since $\sigma_0$ and $A_0$ are fixed, \eqref{B:eq:rate_arithmetic} gives $\widetilde\varepsilon_n/\sigma_0\to0$ and $n\widetilde\varepsilon_n^2/\sigma_0^2\to\infty$. The identity $\norm{\mu_f-\mu_{f_0}}_2^2=n\norm{f-f_0}_n^2$ and Corollary~\ref{B:cor:KL_mass} give
\[
\overline\Pi_n\left\{
\norm{\mu-\mu_{f_0}}_2^2\le n\widetilde\varepsilon_n^2,
\quad
|\sigma^2-\sigma_0^2|\le\sigma_0\widetilde\varepsilon_n
\right\}
\ge e^{-Cn\widetilde\varepsilon_n^2/\sigma_0^2}
\]
for a fixed $C$. Both $C$ and $A_0$ are fixed before the sieve cutoffs are chosen.

\medskip
\noindent\textbf{Tail bound and sieve choice.} Choose a fixed $E$ large enough for the denominator bound in the proof of Theorem~3 of \cite{Jeong2025}. As noted after Proposition~\ref{B:prop:jeong_contraction}, this choice depends on the local mass and local Gaussian bounds, not on the eventual entropy constant. Fix $C_{\mathrm{tail}}\ge E A_0^2/\sigma_0^2$. Proposition~\ref{B:prop:sieve_tail} then provides fixed $M_K,A_\mu,A_\sigma$ and a sieve sequence satisfying
\[
\overline\Pi_n(\overline\Theta_n^c)
\le e^{-En\widetilde\varepsilon_n^2/\sigma_0^2}.
\]

\medskip
\noindent\textbf{Entropy and contraction.} For this selected sieve, Corollary~\ref{B:cor:joint_entropy} gives a finite constant $D$ such that
\[
\log N(\widetilde\varepsilon_n,\overline\Theta_n,d_G)
\le Dn\widetilde\varepsilon_n^2/\sigma_0^2.
\]
The constant $D$ may depend on the selected cutoffs. We choose the final contraction multiplier only after these choices. The three conditions in Proposition~\ref{B:prop:jeong_contraction} now hold simultaneously, giving contraction in $d_G$ for a sufficiently large fixed multiplier.

By the posterior identity above and $\norm{f-f_0}_n+|\sigma-\sigma_0|\le\sqrt2\,d_G((\mu_f,\sigma),(\mu_{f_0},\sigma_0))$, this proves \eqref{B:eq:main_contraction}, after absorbing $A_0$ and $\sqrt2$ into $M$. For sufficiently large $n$, $|\sigma-\sigma_0|\le M\epsn$ implies $\sigma+\sigma_0\le2\sigma_0+1$. Hence $|\sigma^2-\sigma_0^2|\le(2\sigma_0+1)|\sigma-\sigma_0|$ on that event, and the variance-scale conclusion follows by enlarging the fixed multiplier.
\end{proof}

\endgroup

\clearpage
\section{Posterior Contraction Lower Bound for Axis-Aligned BART}\label{C:root}
\begingroup
\renewcommand{\R}{\mathbb{R}}
\renewcommand{\E}{\mathbb{E}}
We use the fixed-design Gaussian model and notation of Supplement~\ref{B:root}.
Let $\Pi_n$ be the RoBART posterior and let $\Pi_n^{\B}$ be the posterior
under the same tree, leaf, noise, and grid specifications with
$Q_t=I_p$ deterministically for every tree. Both priors use the same fixed
$T$, $0<\nu<1$, $\tau^2,a_\sigma,b_\sigma>0$, and deterministic integer
budget $b_n\ge1$. Thus split attempts have probability
$\nu^{\ell+1}$ at depth $\ell$, coordinate choices are uniform over all
$p$ coordinates, and cutpoints are uniform among admissible candidates.
The grid uses the full training design's empirical range, rotated for
RoBART; nodewise admissibility is its strict-interior restriction.
There is no minimum leaf occupancy, and equality is routed to the right.
Write $\mathbb P_n^{\B}$ and $\mathbb P_n^{\B}(\cdot\mid Y^{(n)})$
for the BART generating prior and posterior. The complete states
$(\cT_t,V_t,C_t)$, leaf vectors $M_t$, and total leaf count
$K_+=\sum_{t=1}^T K_t$ retain their meanings from Supplement~\ref{B:root}.

The following restates Theorem~\Mainref{thm:bart-lower}
with its design and grid conditions made explicit.

\begin{Ctheorem}[Posterior contraction lower bound for axis-aligned BART]
\label{C:thm:bart-lower}
Fix integers $p\ge2$ and $T\ge1$, $0<\alpha<1$, and $\lambda,C_h,\sigma_0>0$.
Let $\mathcal X=[0,1]^p$. For $m\ge1$ and $n=2^{mp}$, let the fixed design enumerate
\begin{align*}
\mathcal X_n=
\left\{\left(\frac{i_1+1/2}{2^m},\ldots,
                    \frac{i_p+1/2}{2^m}\right):
       i_1,\ldots,i_p\in\{0,\ldots,2^m-1\}\right\}\subset[0,1]^p.
\end{align*}
Choose $K_{n,1}=2^{L_{n,1}}\asymp(n\lambda^2/\log n)^{1/(2\alpha+1)}$
with $L_{n,1}\in\mathbb N$, and assume
$K_{n,1}\le b_n+1$ and $\log b_n\le C_b\log n$
for a fixed $C_b<\infty$ and all sufficiently large $n$.

There exist $Q_0\in\SO(p)$ and
$h_0:[-\sqrt p,\sqrt p]\to\R$, both independent of $n$, such that
$\norm{h_0}_\infty\le C_h$ and
$|h_0(u)-h_0(v)|\le\lambda|u-v|^\alpha$ on this interval,
and the fixed truth $f_0(x)=h_0\{\pi_{\{1\}}(Q_0x)\}$ on $[0,1]^p$
has the following properties. Set
\begin{align*}
\varepsilon_n=\lambda^{1/(2\alpha+1)}
\left(\frac{\log n}{n}\right)^{\alpha/(2\alpha+1)},
\qquad
\varepsilon_n^{\B}=\left(\frac{\log n}{n}\right)^{\alpha/(2\alpha+p)}.
\end{align*}
For some fixed $c>0$,
\begin{equation}
\E_0\Pi_n^{\B}\bigl(\norm{f-f_0}_n\le c\varepsilon_n^{\B}\mid Y^{(n)}\bigr)
\longrightarrow0.
\label{C:eq:lb-quantitative}
\end{equation}
On the same design and at the same truth, for a sufficiently large fixed $M_0$,
\begin{equation}
\E_0\Pi_n\bigl(\norm{f-f_0}_n+|\sigma-\sigma_0|
                   >M_0\varepsilon_n\mid Y^{(n)}\bigr)
\longrightarrow0.
\label{C:eq:lb-robart}
\end{equation}
Since $\varepsilon_n=o(\varepsilon_n^{\B})$, the BART assertion also holds
with $c\varepsilon_n^{\B}$ replaced by $M\varepsilon_n$ for every fixed
$M>0$. All limits are taken as $m\to\infty$ along $n=2^{mp}$.
\end{Ctheorem}

The budget assumptions are Supplement~\ref{B:root}'s \textnormal{(D2)} for
$R=d_1=1$; $b_n=n$ is a sufficient choice, not a requirement.

\begin{proof}
Constants may depend on the fixed truth, prior, and grid parameters,
but not on $n$.

\medskip
\noindent\textbf{Construction of a fixed truth.}
Choose an integer $J\ge1$ sufficiently large that
\[
\left(\frac{\pi}{2}\right)^p
\sum_{v=1}^{\infty}2^{-Jv(p-\alpha)}
\le\frac12\left(\frac2\pi\right)^p.
\]
For $\kappa>0$ to be chosen, define on $\R$
\[
h_0(u)=\kappa\sum_{s=1}^{\infty}2^{-Js\alpha}
\cos\left(2\pi 2^{Js}\sqrt p\,u-\frac{p\pi}{2}\right).
\]
The series converges absolutely and uniformly. For $0<d=|u-v|\le1$,
\[
|h_0(u)-h_0(v)|
\le C\kappa\left\{
 d\sum_{2^{Js}\le d^{-1}}2^{Js(1-\alpha)}
 +\sum_{2^{Js}>d^{-1}}2^{-Js\alpha}
\right\}
\le C\kappa d^\alpha.
\]
Boundedness handles $d>1$. Taking $\kappa$ sufficiently small gives the
required H\"older and sup-norm bounds on $[-\sqrt p,\sqrt p]$.
Let $a=p^{-1/2}(1,\ldots,1)^\top$ and complete $a^\top$ to the rows of
$Q_0\in\SO(p)$, reversing the sign of another row if necessary.
Then the resulting fixed truth is
\begin{equation}
f_0(x)=\kappa\sum_{s=1}^{\infty}2^{-Js\alpha}
\cos\left(2\pi 2^{Js}\sum_{\ell=1}^p x_\ell-\frac{p\pi}{2}\right).
\label{C:eq:lb-truth-original}
\end{equation}

\medskip
\noindent\textbf{Empirical tensor-Haar coefficients.}
Define
\[
\psi(u)=\1_{[0,1/2)}(u)-\1_{[1/2,1)}(u),
\qquad
\psi_{j,k}(x)=2^{jp/2}\prod_{\ell=1}^p
                         \psi(2^j x_\ell-k_\ell),
\]
for integers $j\ge0$ and $k\in\mathcal I_j=\{0,\ldots,2^j-1\}^p$, and write
$\langle u,v\rangle_n=n^{-1}\sum_i u(x_{i,n})v(x_{i,n})$.
For $j\le m-1$, these functions are empirically orthonormal:
their supports are disjoint, each contains $2^{p(m-j)}$ observations,
and the squared function value on its support is $2^{jp}$.

For even $N\ge2$, pairing the two halves of the sum gives
\begin{align*}
D_N(t)
&:=\frac1N\sum_{r=0}^{N-1}
\psi\left(\frac{r+1/2}{N}\right)
\exp\left(2\pi\mathrm{i}t\frac{r+1/2}{N}\right)\\
&=\frac{1-e^{\pi\mathrm{i}t}}{N}
\sum_{r=0}^{N/2-1}\exp\left(2\pi\mathrm{i}t\frac{r+1/2}{N}\right).
\end{align*}
Consequently,
\[
\begin{aligned}
D_N(1)&=\frac{2\mathrm{i}}{N\sin(\pi/N)},
&\qquad |D_N(1)|&\ge\frac2\pi,\\
D_N(2^v)&=0\quad(v\ge1),
& |D_N(t)|&\le\frac{\pi t}{2}\quad(0<t<1),
\end{aligned}
\]
where $v$ is an integer. The last bound uses
$|1-e^{\pi\mathrm{i}t}|\le\pi t$.
For $j=Js\le m-1$ and $N=2^{m-j}$, tensorization yields
\[
\left\langle
\exp\left(2\pi\mathrm{i}2^j t\sum_{\ell=1}^p x_\ell\right),
\psi_{j,k}\right\rangle_n
=2^{-jp/2}\exp\left(2\pi\mathrm{i}t\sum_{\ell=1}^p k_\ell\right)
D_N(t)^p.
\]
The $s$th summand of \eqref{C:eq:lb-truth-original} has $t=1$.
Its spatial phase is one, and $e^{-p\pi\mathrm{i}/2}$ cancels the phase
of $D_N(1)^p$. It therefore contributes at least
$\kappa(2/\pi)^p2^{-j(\alpha+p/2)}$ to every coefficient.
All higher-frequency summands vanish since $t=2^{Jv}$.
The absolute sum of the lower-frequency contributions is bounded by
\begin{align*}
\kappa 2^{-jp/2}\sum_{v=1}^{s-1}
  2^{-(j-Jv)\alpha}\left(\frac\pi2\,2^{-Jv}\right)^p
&=\kappa\left(\frac\pi2\right)^p
  2^{-j(\alpha+p/2)}\sum_{v=1}^{s-1}2^{-Jv(p-\alpha)}\\
&\le\frac\kappa2\left(\frac2\pi\right)^p
       2^{-j(\alpha+p/2)}.
\end{align*}
Absolute convergence permits termwise evaluation in these finite sums.
Thus, for every positive multiple $j$ of $J$ with $j\le m-1$,
\begin{equation}
\langle f_0,\psi_{j,k}\rangle_n
\ge c_H2^{-j(\alpha+p/2)},
\qquad c_H=\frac\kappa2\left(\frac2\pi\right)^p>0,
\qquad k\in\mathcal I_j.
\label{C:eq:lb-dense-coefficients}
\end{equation}

\medskip
\noindent\textbf{Approximation lower bound for a forest.}
For an axis-aligned rectangle $B$, the coefficient
$\langle\1_B,\psi_{j,k}\rangle_n$ factors into univariate coefficients.
An interval indicator can vary on at most two level-$j$ dyadic intervals;
on every other interval its coefficient vanishes because the two halves
contain equal numbers of design points. Hence a rectangle contributes to
at most $2^p$ coefficients. Any forest with at most $K$ total leaves is
$g=\sum_{\ell=1}^K b_\ell\1_{B_\ell}$, allowing overlapping rectangles
and unrestricted signed weights, so
\[
\#\{k:\langle g,\psi_{j,k}\rangle_n\ne0\}
\le\sum_{\ell=1}^K
\#\{k:\langle\1_{B_\ell},\psi_{j,k}\rangle_n\ne0\}
\le2^pK.
\]
When $2^{jp}\ge2^{p+1}K$, at least half of the coefficients in
\eqref{C:eq:lb-dense-coefficients} are zero for $g$.
Bessel's inequality gives
\[
\norm{f_0-g}_n^2
\ge\sum_{\substack{k\in\mathcal I_j:\
                   \langle g,\psi_{j,k}\rangle_n=0}}
                 |\langle f_0,\psi_{j,k}\rangle_n|^2
\ge\frac{c_H^2}{2}\,2^{-2j\alpha}.
\]
Choose the smallest positive multiple $j$ of $J$ such that
$2^{jp}\ge2^{p+1}K$; then $2^{jp}\asymp K$.
The largest multiple of $J$ not exceeding $m-1$ is at least $m-J$,
so the selected $j$ is admissible whenever $K\le c_{\mathcal X}n$
for a sufficiently small fixed $c_{\mathcal X}>0$ and large $m$.
Writing $\Faxis(K)$ for all axis-aligned constant-leaf forests with
at most $K$ total leaves, we obtain
\begin{equation}
\inf_{g\in\Faxis(K)}\norm{g-f_0}_n
\ge c_AK^{-\alpha/p},\qquad 1\le K\le c_{\mathcal X}n.
\label{C:eq:lb-approximation}
\end{equation}
This class allows arbitrary numerical cutpoints and leaf values,
so the bound applies to the declared empirical-range prior.

\medskip
\noindent\textbf{Local prior mass.}
Put $q_n=n(\varepsilon_n^{\B})^2$ and
$l_n=\lfloor\log_2(n/\log n)/(2\alpha+p)\rfloor$.
Choose a complete tree with depth $L_n^{\B}=p l_n$ and
$K_n^{\B}=2^{p l_n}$ leaves, cycling through the $p$ coordinates.
Then
\begin{equation}
K_n^{\B}\asymp(n/\log n)^{p/(2\alpha+p)},\qquad
2^{-\alpha l_n}\asymp\varepsilon_n^{\B},\qquad
K_n^{\B}\log n\asymp q_n\longrightarrow\infty.
\label{C:eq:lb-local-scales}
\end{equation}
The budget is sufficient for the $l_n$ splits in each coordinate, since
\[
\frac{2^{l_n}}{b_n+1}\le\frac{2^{l_n}}{K_{n,1}}
\lesssim(n/\log n)^{1/(2\alpha+p)-1/(2\alpha+1)}\longrightarrow0.
\]
Use the lower median admissible candidate at every split. The integer-cell
recursion $N'\mapsto(\lfloor N'/2\rfloor,\lceil N'/2\rceil)$ from
Proposition~\ref{B:prop:dynamic_uniform_sparse}(ii) of Supplement~\ref{B:root} gives, after $a$ coordinate splits,
between $\lfloor(b_n+1)/2^a\rfloor$ and
$\lceil(b_n+1)/2^a\rceil$ cells in each coordinate interval.
Since $2^{l_n}\le b_n+1$, each requested split starts with at least
two cells. The final coordinate widths are bounded by
\[
(1-2^{-m})\frac{\lceil(b_n+1)/2^{l_n}\rceil}{b_n+1}
\le2^{1-l_n}.
\]
These cells cover the empirical-range box, and empty leaves are allowed.
The bound
$|f_0(x)-f_0(y)|\le\lambda p^{\alpha/2}\norm{x-y}_\infty^\alpha$
therefore gives an approximating tree $\widehat g_n^{\B}$ with
$\norm{\widehat g_n^{\B}-f_0}_n\le C\varepsilon_n^{\B}$
by taking $f_0$ at a point in each leaf closure.
Denote its ordered leaf center by $\beta_n^{\B}$; then
$\norm{\beta_n^{\B}}_\infty\le C_h$.

Use this as the first tree and take the remaining $T-1$ trees to be
root-only. Let $\mathfrak t_n^{\B}$ be this prescribed tuple of
complete discrete states. Each prescribed coordinate and cut contributes
at least $(pb_n)^{-1}$, and each terminal factor is at least $1-\nu$.
Since $\sum_{d=0}^{L_n^{\B}-1}(d+1)2^d
=(L_n^{\B}-1)K_n^{\B}+1$,
\begin{equation}
\begin{aligned}
&\mathbb P_n^{\B}\{(\cT_t,V_t,C_t)_{t=1}^T=\mathfrak t_n^{\B}\}\\
&\quad\ge(1-\nu)^{K_n^{\B}+T-1}
(pb_n)^{-(K_n^{\B}-1)}
\nu^{(L_n^{\B}-1)K_n^{\B}+1}
\ge e^{-C K_n^{\B}\log n},
\end{aligned}
\label{C:eq:lb-template-mass}
\end{equation}
using $L_n^{\B}=O(\log n)$ and $\log b_n\le C_b\log n$.

Within the prescribed-state event, define $A_n^{\B}$ by
$\norm{M_1-\beta_n^{\B}}_\infty\le\varepsilon_n^{\B}/(2T)$,
$|\mu_t|\le\varepsilon_n^{\B}/(2T)$ for $t=2,\ldots,T$, and
$|\sigma^2-\sigma_0^2|\le\sigma_0\varepsilon_n^{\B}$.
Here $M_1$ and $\beta_n^{\B}$ use the same left-to-right leaf order,
and $\mu_t$ is the unique leaf mean of the $t$th root-only tree.
The coefficient comparisons are evaluated only on that state;
when $T=1$ the remaining-tree restrictions are absent.
The Gaussian densities have a common positive lower bound around these
bounded centers, and the inverse-gamma density is positive near
$\sigma_0^2$. Each scalar restriction therefore has probability at least
$c\varepsilon_n^{\B}$ for large $n$. Independence and
\eqref{C:eq:lb-local-scales}--\eqref{C:eq:lb-template-mass} give
\begin{equation}
\mathbb P_n^{\B}(A_n^{\B})\ge e^{-C_Pq_n}.
\label{C:eq:lb-local-mass}
\end{equation}
On $A_n^{\B}$ the total leaf perturbation is at most
$\varepsilon_n^{\B}/2$ at every design point. Thus
$\norm{f-f_0}_n\le C_f\varepsilon_n^{\B}$ and
$|\sigma-\sigma_0|
=|\sigma^2-\sigma_0^2|/(\sigma+\sigma_0)\le\varepsilon_n^{\B}$.

\Needspace{8\baselineskip}
\medskip
\noindent\textbf{A marginal likelihood lower bound.}
For a generating-space state $\theta$ inducing $(f,\sigma)$, let
$L_n(\theta)=p_{f,\sigma}^{(n)}(Y^{(n)})/p_0^{(n)}(Y^{(n)})$
and $D_n=\int L_n(\theta)\,d\mathbb P_n^{\B}(\theta)$, where the
$p$'s denote densities. The normalizer is positive and finite because
the likelihood is bounded by $(2\pi)^{-n/2}\sigma^{-n}$ and the
inverse-gamma prior has finite negative moments.
The local error bounds and Lemma~\ref{B:lem:gaussian_kl} of Supplement~\ref{B:root} imply
\[
\sup_{\theta\in A_n^{\B}}
\left[\KL(P_0^{(n)},P_{f,\sigma}^{(n)})
+\operatorname{Var}_0\{\log L_n(\theta)\}\right]
\le C_Gq_n.
\]
Let $\widetilde{\mathbb P}_n^{\B}$ be the prior restricted to
$A_n^{\B}$ and normalized, and put
$Z_n=\int\log L_n(\theta)\,d\widetilde{\mathbb P}_n^{\B}(\theta)$.
Then $\E_0Z_n\ge-C_Gq_n$ and, by Jensen's inequality,
\[
\operatorname{Var}_0(Z_n)
\le\int\operatorname{Var}_0\{\log L_n(\theta)\}\,
 d\widetilde{\mathbb P}_n^{\B}(\theta)
\le C_Gq_n.
\]
Moreover, $\log D_n\ge\log\mathbb P_n^{\B}(A_n^{\B})+Z_n$.
Chebyshev's inequality and \eqref{C:eq:lb-local-mass}, with
$C_D=C_P+C_G+1$, therefore give
\begin{equation}
P_0^{(n)}(\mathcal D_n^c)\le C_G/q_n\longrightarrow0,
\qquad \mathcal D_n=\{D_n\ge e^{-C_Dq_n}\}.
\label{C:eq:lb-denominator}
\end{equation}

\medskip
\noindent\textbf{Posterior complexity and the lower bound.}
The split-attempt proof of Lemma~\ref{B:lem:tree_tail} of Supplement~\ref{B:root} applies to the
axis-aligned recursion uniformly in $n$ and the finite grid budget.
A union bound over the fixed number $T$ of trees gives
\begin{equation}
\mathbb P_n^{\B}(K_+>k)\le C_2\exp(-c_2k\log k)
\label{C:eq:lb-forest-tail}
\end{equation}
for fixed $c_2,C_2>0$ and all sufficiently large $k$.
Set $\bar K_n^{\B}=\lceil C_Kq_n/\log n\rceil$.
Then $\log\bar K_n^{\B}\asymp\log n$,
$\bar K_n^{\B}=o(n)$, and
$(\bar K_n^{\B})^{-\alpha/p}\asymp\varepsilon_n^{\B}$.
Choose $C_K$ after $C_D$ so that \eqref{C:eq:lb-forest-tail} yields
$\mathbb P_n^{\B}(K_+>\bar K_n^{\B})\le e^{-(C_D+2)q_n}$.
Since $\E_0L_n(\theta)=1$, Tonelli's theorem and
\eqref{C:eq:lb-denominator} give
\begin{equation}
\begin{aligned}
\E_0\mathbb P_n^{\B}(K_+>\bar K_n^{\B}\mid Y^{(n)})
&\le P_0^{(n)}(\mathcal D_n^c)
 +e^{C_Dq_n}\E_0\int_{\{K_+>\bar K_n^{\B}\}}
 L_n(\theta)\,d\mathbb P_n^{\B}(\theta)\\
&=P_0^{(n)}(\mathcal D_n^c)
 +e^{C_Dq_n}\mathbb P_n^{\B}(K_+>\bar K_n^{\B})\\
&\le C_G/q_n+e^{-2q_n}\longrightarrow0.
\end{aligned}
\label{C:eq:lb-posterior-size}
\end{equation}
By \eqref{C:eq:lb-approximation}, every state with
$K_+\le\bar K_n^{\B}$ has error at least
$c_A(\bar K_n^{\B})^{-\alpha/p}$ for large $n$.
Taking $c>0$ sufficiently small gives the generating-space inclusion
\[
\{\norm{f-f_0}_n\le c\varepsilon_n^{\B}\}
\subseteq\{K_+>\bar K_n^{\B}\}.
\]
The induced posterior bound \eqref{C:eq:lb-quantitative} follows from
\eqref{C:eq:lb-posterior-size}; no uniqueness of the forest representation
is needed. Since $p\ge2$, $\varepsilon_n=o(\varepsilon_n^{\B})$,
which gives the assertion for every fixed $M\varepsilon_n$ as well.

\medskip
\noindent\textbf{Application of the RoBART theorem.}
The constructed truth satisfies the bounded-design and H\"older
conditions of Supplement~\ref{B:root} with $M_X=\sqrt p$, $R=d_1=1$,
$S_1=\{1\}$, $\lambda_1=\lambda$, and $\alpha_{11}=\alpha$.
Its active range is
$R_{n,1,1}(Q_0)=\sqrt p(1-2^{-m})\ge\sqrt p/2$, so
\textnormal{(D1)} holds. Condition \textnormal{(D2)} is assumed,
and the remaining prior specifications coincide. Theorem~\ref{B:thm:main} of
Supplement~\ref{B:root} therefore gives \eqref{C:eq:lb-robart}.
\end{proof}

\endgroup

\clearpage
\section{Experimental Settings}\label{D:root}
\paragraph{Benchmark.}
The design in Section~\Mainref{sec:benchmark} is fixed across thirty
replications; methods share inputs and noise within each replication.
All methods use 200 trees and 1,000 unthinned draws after 9,000 burn-in
iterations. RoBART starts from identity rotations and uses
$K_{\rm rot}=5$, $\gamma=0.08$ and $\theta_{\max}=0.35\gamma/0.15$,
with joint cutpoint retuning and support-expansion probability $0.5$.
BART keeps the identity. Both use 100 cutpoints per coordinate,
leaf shrinkage $k=2$, noise-prior parameters \texttt{sigdf=3} and
\texttt{sigquant=0.9}, and polynomial split probability
$0.95/(1+\ell)^2$ at depth $\ell$. The noise scale is estimated.

Soft BART uses training-only empirical-distribution preprocessing,
$k=2$ and temperature 1. Noise, leaf-scale, splitting-probability,
$\alpha$ and bandwidth updates are enabled; \texttt{beta}, \texttt{gamma},
\texttt{tau\_mean} updates and tree caching are disabled. Oblique BART and
ridgeBART use training-based coordinatewise affine scaling to $[-1,1]$,
applied to test inputs without clipping. Oblique BART uses
\texttt{prob\_aa=0.5} with adaptation, \texttt{alpha\_dp=5},
\texttt{phi\_option=7}, and \texttt{a\_theta=b\_theta=0}.
ridgeBART uses all predictors for splitting and smoothing, one ReLU basis,
\texttt{p\_change=0.2}, \texttt{sparse=FALSE}, \texttt{rho\_alpha=400},
\texttt{rho\_nu=3}, and \texttt{rho\_lambda} equal to the median of
$\chi^2_3$ divided by 3.

\paragraph{Illustrative fits.}
Figure~\Mainref{fig:motivation} compares the same five methods using
200 trees, 4,000 burn-in iterations, 1,000 unthinned draws, and a
$151\times151$ prediction grid. The data seed is 1; sampler seeds are
12, 42, 54, 38 and 24 for BART, Soft BART, ridgeBART, Oblique BART and
RoBART. RoBART uses $\theta_{\max}=0.35$. Soft BART receives the original
inputs and updates only the noise scale, leaf scale and bandwidth.
ridgeBART uses uniform cuts and scales inputs by the largest absolute
training entry, clipping only prediction inputs to $[-1,1]$.
Oblique BART uses uniform cuts and \texttt{prob\_aa=0} without adaptation.

\end{document}